\documentclass{article}
\usepackage{ijcai24}

\usepackage{times}
\usepackage{soul}
\usepackage{url}
\usepackage[hidelinks]{hyperref}
\usepackage[utf8]{inputenc}
\usepackage[small]{caption}
\usepackage{graphicx}
\usepackage{amsmath,amssymb}
\usepackage{amsthm}
\usepackage{booktabs}
\usepackage{algorithm}
\usepackage{algorithmic}
\usepackage[switch]{lineno}

\usepackage{subcaption}

\newtheorem{proposition}{Proposition}
\newtheorem{lemma}{Lemma}
\newtheorem{definition}{Definition}

\usepackage{xcolor}

\allowdisplaybreaks

\title{Capturing Knowledge Graphs and Rules with Octagon Embeddings}

\author{
Victor Charpenay$^1$
\and
Steven Schockaert$^2$\\
\affiliations
$^1$Mines Saint-Etienne, UMR 6158 LIMOS, Saint-\'Etienne, France\\
$^2$Cardiff University, Cardiff, UK\\
\emails
victor.charpenay@emse.fr,
schockaerts1@cardiff.ac.uk
}

\begin{document}

\maketitle

\begin{abstract}
Region based knowledge graph embeddings represent relations as geometric regions. This has the advantage that the rules which are captured by the model are made explicit, making it straightforward to incorporate prior knowledge and to inspect learned models. Unfortunately, existing approaches are severely restricted in their ability to model relational composition, and hence also their ability to model rules, thus failing to deliver on the main promise of region based models. With the aim of addressing these limitations, we investigate regions which are composed of axis-aligned octagons. Such octagons are particularly easy to work with, as intersections and compositions can be straightforwardly computed, while they are still sufficiently expressive to model arbitrary knowledge graphs. Among others, we also show that our octagon embeddings can properly capture a non-trivial class of rule bases. Finally, we show that our model achieves competitive experimental results.
\end{abstract}

\section{Introduction}
Knowledge graphs, i.e.\ sets of (entity, relation, entity) triples,  have become one of the most popular frameworks for knowledge representation, with applications ranging from search \cite{DBLP:journals/ftir/ReinandaMR20} and recommendation \cite{DBLP:journals/tkde/GuoZQZXXH22} to natural language processing \cite{DBLP:conf/ijcnlp/SchneiderSVGSM22} and computer vision \cite{DBLP:conf/cvpr/MarinoSG17}. Their popularity has spurred an extensive line of work dedicated to representation learning on knowledge graphs. Most works in this area focus on the paradigm of knowledge graph embedding \cite{DBLP:conf/aaai/BordesWCB11}, aiming to learn a vector representation $\mathbf{e}$ for each entity $e$ and a scoring function $s_r$ for each relation $r$ such that $s_r(\mathbf{e},\mathbf{f})$ reflects the likelihood that $(e,r,f)$ is a valid triple, i.e.\ that entity $e$ is in relation $r$ with entity $f$. 

Knowledge graph embeddings are intended to capture semantic regularities, making it possible to predict plausible triples that were missing from the original knowledge graph. Consider, for instance, the seminal TransE model \cite{DBLP:conf/nips/BordesUGWY13}, which uses scoring functions of the form $s_r(\mathbf{e},\mathbf{f}) = -d(\mathbf{e}+\mathbf{r},\mathbf{f})$, where $\mathbf{r}\in \mathbb{R}^n$ is an embedding of the relation $r$, $\mathbf{e},\mathbf{f}\in\mathbb{R}^n$ are entity embeddings of the same dimension and $d$ is a distance metric. Note that $s_r(\mathbf{e},\mathbf{f})$ achieves its maximal value of $0$ iff $\mathbf{e}+\mathbf{r}=\mathbf{f}$. We can thus say that the embedding fully supports the triple $(e,r,f)$ when $s_r(\mathbf{e},\mathbf{f})=0$. If $\mathbf{r_1}+\mathbf{r_2}=\mathbf{r_3}$ we have that $s_{r_1}(\mathbf{e},\mathbf{f})=0$ and  $s_{r_2}(\mathbf{f},\mathbf{g})=0$ together imply $s_{r_3}(\mathbf{e},\mathbf{g})=0$. In this sense, we can say that the embedding captures the following rule:
\begin{align}\label{eqExampleRule}
r_1(X,Y) \wedge r_2(Y,Z) \rightarrow r_3(X,Z)
\end{align}
This correspondence between knowledge graph embeddings and symbolic rules is appealing. For instance, if we already know  that \eqref{eqExampleRule} is valid, we can impose $\mathbf{r_1}+\mathbf{r_2}=\mathbf{r_3}$ when learning the embedding. We can also use this correspondence to inspect which semantic dependencies a learned embedding is capturing. 
Unfortunately, TransE has some inherent limitations, which mean that certain knowledge graphs cannot be faithfully captured \cite{DBLP:conf/aaai/WangGL18}. This has been addressed in more recent models such as ComplEx \cite{DBLP:journals/jmlr/TrouillonDGWRB17}, ConvE \cite{DBLP:conf/aaai/DettmersMS018} and TuckER \cite{DBLP:conf/emnlp/BalazevicAH19}, to name just a few, but while these models often perform better on the task of link prediction, the connection between their parameters and the captured semantic dependencies is considerably more opaque.  

Region based knowledge graph embedding models \cite{DBLP:conf/kr/Gutierrez-Basulto18,DBLP:conf/nips/AbboudCLS20,DBLP:conf/iclr/0002S23} aim to get the best of both worlds, increasing the expressivity of TransE while maintaining an explicit correspondence between model parameters and semantic dependencies. They represent each relation $r$ as a geometric region $X_r \subseteq \mathbb{R}^{2n}$. We say that a triple $(e,r,f)$ is captured by the embedding if $\mathbf{e}\oplus \mathbf{f} \in X_r$, where  $\oplus$ denotes vector concatenation. To characterise semantic dependencies, we can then exploit the fact that the intersection, subsumption and composition of relations can be naturally modelled in terms of their embeddings $X_r$. For instance, we say that the embedding captures the rule $r_1(X,Y) \wedge r_2(X,Y) \rightarrow r_3(X,Y)$ iff $X_{r_1}\cap X_{r_2} \subseteq X_{r_3}$. To model rules of the form \eqref{eqExampleRule}, we can characterise relational composition as follows:
\begin{align}\label{eqDefComposition}
X_r \diamond X_s = \{\mathbf{e}\oplus \mathbf{g} \,|\, \exists \mathbf{f}\in\mathbb{R}^n\,.\, \mathbf{e}\oplus \mathbf{f}\in X_r \wedge \mathbf{f}\oplus \mathbf{g}\in X_s\}
\end{align}
We then say that \eqref{eqExampleRule} is captured iff $X_{r_1}\diamond X_{r_2} \subseteq X_{r_3}$. Rules of the form \eqref{eqExampleRule}, which are sometimes referred to as closed path rules, play an important role in link prediction \cite{DBLP:conf/ijcai/MeilickeCRS19}. When designing a region based model, it is thus important that the composition of regions can be straightforwardly characterised. However, existing approaches are severely limited in this respect. For instance, BoxE \cite{DBLP:conf/nips/AbboudCLS20} cannot model relational composition at all, while ExpressivE \cite{DBLP:conf/iclr/0002S23} uses parallelograms, which are not closed under composition.


To address these concerns, our aim in this paper is to develop a region based model which is as simple as possible, while (i) still being expressive enough to capture arbitrary knowledge graphs and (ii) using regions which are closed under intersection and composition. Based on these desiderata, we arrive at a model which relies on axis-aligned octagons (with all angles fixed at 45 degrees).
We show that despite their simplicity, the proposed octagon embeddings are sufficiently expressive to properly capture a large class of rule bases. This is an important property, among others because it means that we can inject prior knowledge, in the form of a given rule base, without unintended consequences. Our result is considerably more general than what is possible with BoxE, which is not able to capture any closed path rules, and more general than what is known about ExpressivE. Moreover, because compositions and intersections of octagons can be straightforwardly computed, octagon embeddings are considerably more practical than existing alternatives. While our main focus is on better understanding the expressivity of region based models, we have also empirically evaluated the proposed octagon embeddings.
We found octagons to achieve results close to the current state-of-the-art.
This demonstrates that learning octagon embeddings is a promising strategy, especially in contexts where both knowledge graphs and rules need to be modelled.\footnote{This paper is to appear in the proceedings of IJCAI 24.}

\section{Preliminaries}
Let $\mathcal{E}$ and $\mathcal{R}$ be sets of entities and relations respectively. We consider knowledge graph embeddings in which each entity $e\in \mathcal{E}$ is represented by a vector $\mathbf{e}\in\mathbb{R}^n$ and each relation $r\in\mathcal{R}$ is represented by a region $X_r\subseteq \mathbb{R}^{2n}$. We refer to such representations as region-based knowledge graph embeddings.
We say that a triple $(e,r,f)$ is captured or supported by a given embedding if $\mathbf{e}\oplus \mathbf{f}\subseteq X_r$. This notion of support allows us to unambiguously associate a knowledge graph with a given geometric embedding, which in turn allows us to study the expressivity of knowledge graph embedding models. In practice, we typically use regions with soft boundaries. This makes learning easier and it is better aligned with the fact that link prediction is typically treated as a ranking problem rather than a classification problem. We will return to the issue of learning regions with soft boundaries in Section \ref{secLearningOctagonEmbeddings}.

\begin{figure}
\centering
\includegraphics[height=135pt]{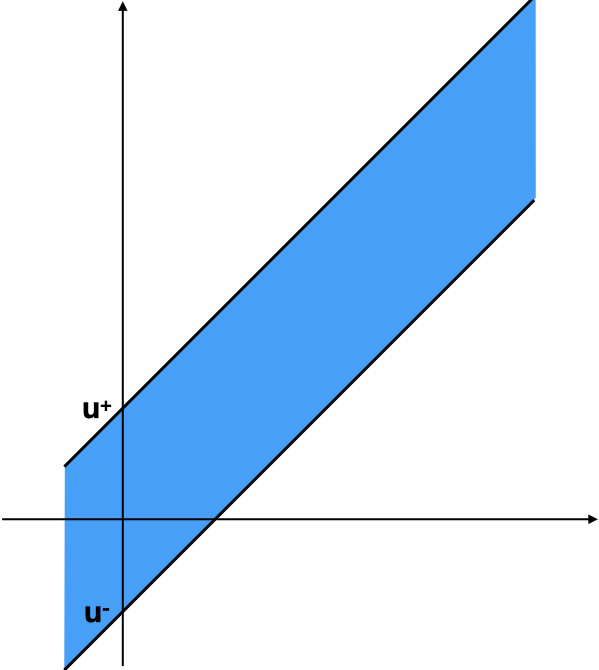}
\includegraphics[height=135pt]{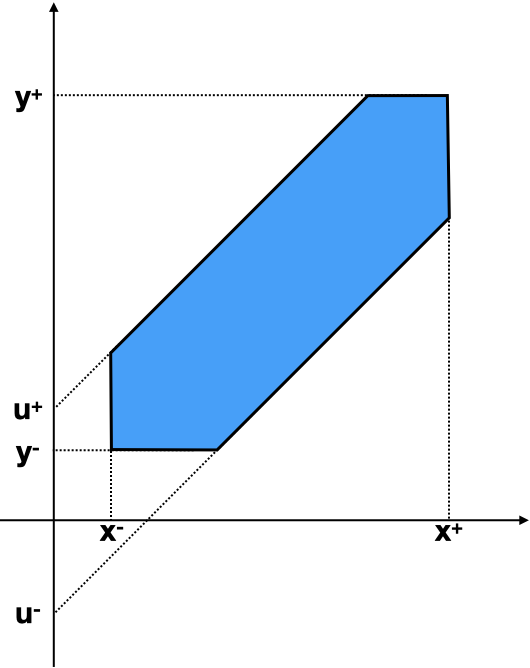}
\caption{Basic region based embeddings: TransE bands (left) and hexagons (right).\label{figTransEHexagon}}
\end{figure}

\begin{figure}
\centering
\includegraphics[width=200pt]{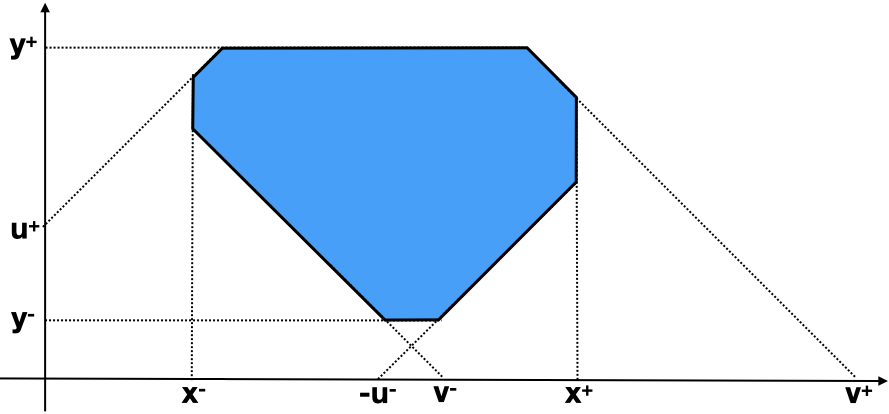}
\caption{Region-based embeddings with octagons.\label{figOctagon}}
\end{figure}
\paragraph{Coordinate-wise region embeddings.}
Region based embeddings were first studied from a theoretical point of view in \cite{DBLP:conf/kr/Gutierrez-Basulto18}. Their setting allowed relations to be represented by arbitrary convex polytopes. Using such regions is not feasible in practice, however, as they require exponentially many parameters. A natural solution is to use regions $X_r$ which can be described using $n$ two-dimensional regions $X^i_r$:
\begin{align}\label{eqCoordinateWise}
X_r \,{=}\, \{(x_1{,}...,x_{n},y_1{,}...,y_n) \,|\, \forall i\in \{1{,}...,n\}.(x_i,y_{i})\in X^i_r\}
\end{align}
We call such region based embeddings \emph{coordinate-wise}. For example, we can view TransE as a coordinate-wise region based embedding model, where 
\begin{align}\label{eqTransEStripes}
X^i_r = \{(x,y) \,|\, u_i^- \leq y-x \leq u_i^+\}
\end{align}
In this case, the two-dimensional regions correspond to unbounded bands, as illustrated in Figure \ref{figTransEHexagon}. ExpressivE \cite{DBLP:conf/iclr/0002S23} is also a coordinate-wise model, where the regions $X^i_r$ are parallelograms. An important drawback of parallelograms is that they are not closed under composition: if $X^i_{r_1}$ and $X^i_{r_2}$ are parallelograms, then it may be that $X^i_{r_1} \diamond X^i_{r_2}$ is not a parallelogram, with $\diamond$ defined as in \eqref{eqDefComposition}. Similarly, parallelograms are not closed under intersection. To address these limitations, we can consider some alternatives. Starting from the TransE bands defined in \eqref{eqTransEStripes}, a natural generalisation is to add domain and range constraints, i.e.\ to use hexagonal regions of the following form, as also illustrated in Figure \ref{figTransEHexagon}:
\begin{align}\label{eqHexagons}
X^i_r = \{(x,y) \,|\, &u_i^- \leq y-x \leq u_i^+, \\ 
&x_i^- \leq x \leq x_i^+, y_i^- \leq y \leq y_i^+\} \notag
\end{align}
Unfortunately, as we will see in Section \ref{secExpressivity}, such hexagons are not sufficiently expressive to capture arbitrary knowledge graphs. As a minimal extension, we consider octagons of the following form, as illustrated in Figure \ref{figOctagon}:
\begin{align}\label{eqOctagons}
X^i_r = \{(x,y) \,|\, &u_i^- \leq y-x \leq u_i^+, v_i^- \leq x+y \leq v_i^+,\\ 
&x_i^- \leq x \leq x_i^+, y_i^- \leq y \leq y_i^+\} \notag
\end{align}
Throughout the paper, we will refer to regions of the form \eqref{eqOctagons} as octagons. However, note that some of these regions are degenerate, having in fact fewer than eight vertices.


\paragraph{Cross-coordinate region embeddings.}
Some region based models go beyond coordinate-wise embeddings. We will refer to them as \emph{cross-coordinate} models. For instance, BoxE \cite{DBLP:conf/nips/AbboudCLS20} uses regions of the following form:
\begin{align}\label{eqBoxE}
X_r = \{&(x_1,\ldots,x_{2n},y_1,\ldots,y_{2n}) \,|\,\\
&\forall i\in\{1,\ldots,n\}\,.\,(x_i,y_{n+i})\,{\in}\, A^i_r \wedge (x_{n+i},y_i)\,{\in}\, B^i_r\}\notag
\end{align}
An advantage of cross-coordinate models is that they can rely on simpler two-dimensional regions. For instance, the regions $A^i_r$ and $B^i_r$ in the BoxE model essentially correspond to TransE bands. While a coordinate-wise model with TransE bands is not sufficiently expressive to capture arbitrary knowledge graphs, by using cross-coordinate comparisons, BoxE can circumvent this limitation. 

\paragraph{Capturing rules.}
The main promise of region based models is that the spatial configuration of the regions reveals the semantic dependencies that are captured. Let us first consider a closed-path rule of the following form:
\begin{align}\label{eqClosedPathRule}
r_1(X_1,X_2) \wedge \ldots \wedge r_k(X_k,X_{k+1}) \rightarrow r(X_1,X_{k+1})
\end{align}
Such a rule can be encoded using relational composition and set inclusion, as follows:
$$
r_1 \circ \ldots \circ r_k \subseteq r
$$
Accordingly, we say that a region based embedding captures this rule iff the following inclusion holds:
$$
X_{r_1} \diamond \ldots \diamond X_{r_k} \subseteq X_r
$$
Other types of rules can be modelled in a similar way. For instance, a rule such as $r_1(X,Y) \wedge r_2(X,Y) \rightarrow r(X,Y)$ is captured by a region based embedding if $X_{r_1}\cap X_{r_2} \subseteq X_r$.

To illustrate how rules can be captured by region based embeddings, first consider a coordinate-wise model with TransE bands, i.e.\ regions of the form \eqref{eqTransEStripes}. We have that $X_{r_1}\diamond \ldots \diamond X_{r_k} \subseteq X_r$ holds iff $\sum_{i=1}^k u_{r_i}^- \geq u_r^-$ and $\sum_{i=1}^k u_{r_i}^+ \leq u_r^+$. This means that every rule of the form \eqref{eqClosedPathRule} can be satisfied. However, the model is not sensitive to the order in which the relations in the body appear. For instance, whenever $r_1(X,Y) \wedge r_2(Y,Z) \rightarrow r(X,Z)$ is satisfied, we also have that $r_2(X,Y) \wedge r_1(Y,Z) \rightarrow r(X,Z)$ is satisfied, which is clearly undesirable. Coordinate-wise models with TransE bands are thus not sufficiently expressive to properly capture rules. 

Now consider a cross-coordinate model such as BoxE, using regions of the form \eqref{eqBoxE}. It holds that $(x_1,\allowbreak \ldots,\allowbreak x_{2n},\allowbreak z_1,\allowbreak \ldots,\allowbreak z_{2n})\in X_{r_1} \diamond X_{r_2}$ iff there are $y_1,\allowbreak \ldots,\allowbreak y_{2n}\in\mathbb{R}$ such that for every $i\in \{1,\ldots,n\}$ we have:
\begin{align*}
(x_i,y_{n+i})&\in A_{r_1}^i &
(x_{n+i},y_i)&\in B_{r_1}^i \\
(y_i,z_{n+i})&\in A_{r_2}^i &
(y_{n+i},z_i)&\in B_{r_2}^i 
\end{align*}
This is the case iff for every $i\in \{1,\ldots,n\}$ we have:
\begin{align*}
(x_i,z_i)&\in A_{r_1}^i \diamond B_{r_2}^i &
(x_{n+i},z_{n+i})&\in B_{r_1}^i \diamond A_{r_2}^i
\end{align*}
While $X_{r_1} \diamond X_{r_2}$ constrains the pairs $(x_i,z_i)$ and $(x_{n+i},z_{n+i})$, any region $X_{r_3}$ rather constrains the pairs $(x_i,z_{n+i})$ and $(x_{n+i},z_{i})$. It is thus not possible to have $X_{r_1} \diamond X_{r_2}\subseteq X_{x_3}$, unless in trivial cases where $X_{r_1}=\emptyset$, $X_{r_2}=\emptyset$, or (if infinite bounds are allowed) $X_{r_3}=\mathbb{R}^{4n}$. In all these trivial cases, the embedding also captures other rules, e.g.\ if $X_{r_1}=\emptyset$ then all rules of the form $r_1(X,Y)\wedge r_4(Y,Z)\rightarrow r_5(X,Z)$ are also satisfied.

ExpressivE is able to capture rules of the form \eqref{eqClosedPathRule}, while avoiding unintended consequences \cite{DBLP:conf/iclr/0002S23}. However, in practice we are typically interested in capturing \emph{sets} of such rules, and it is unclear under which conditions this is possible with ExpressivE. Another drawback of ExpressivE is that compositions and intersections of regions are difficult to compute, 
which, among others, makes checking whether a given rule is satisfied computationally expensive. Moreover, it is not clear how injecting rules into the learning process can then be done in a practical way.

\section{Modelling Relations with Octagons}\label{secModellingWithOctagons}
In this paper, we focus on coordinate-wise models where the two-dimensional regions are octagons. Let us write $X_r = [O_1^r,\ldots,O_n^r]$ to denote that relation $r$ is defined using the octagons $O_1^r,\ldots,O_n^r$, in the sense that $(x_1,\allowbreak \ldots,\allowbreak x_n,\allowbreak y_1,\allowbreak \ldots,\allowbreak y_n)\in X_r$ iff $(x_i,y_i)\in O_i^r$ for each $i\in \{1,\ldots,n\}$. Clearly, if $X_r = [O_1^r,\ldots,O_n^r]$ and $X_s = [O_1^s,\ldots,O_n^s]$ then we have $X_r\cap X_s = [O_1^r \cap O_1^s,\ldots,O_n^r \cap O_n^s]$ and $X_r\diamond X_s = [O_1^r \diamond O_1^s,\allowbreak \ldots,\allowbreak O_n^r \diamond O_n^s]$. To study how properties such as reflexivity, symmetry and transitivity can be satisfied in octagon embeddings, it is thus sufficient to study these properties for individual octagons, which is what we focus on in this section. 

\paragraph{Parameterisation.} Let us write $\mathsf{Octa}(x_i^-,\allowbreak x_i^+,\allowbreak y_i^-,\allowbreak y_i^+,\allowbreak u_i^-,\allowbreak u_i^+,\allowbreak v_i^-,\allowbreak v_i^+)$ to denote the octagon defined in \eqref{eqOctagons}. Note that the eight parameters are not independent. 
For instance, we have that $\mathsf{Octa}(x^-,\allowbreak x^+,\allowbreak y^-,\allowbreak y^+,\allowbreak u^-,\allowbreak u^+,\allowbreak v^-,\allowbreak v^+)=\mathsf{Octa}(x^-,\allowbreak x^+,\allowbreak y^-,\allowbreak y^+,\allowbreak u^-,\allowbreak u^+,\allowbreak \max(v^-,x^-+y^-),\allowbreak v^+)$. Indeed, if $x\geq x^-$ and $y\geq y^-$, we also have $x+y \geq x^- + y^-$, meaning that the bound $x+y\geq v^-$ can be strengthened to $x+y\geq \max(x^-+y^-,v^-)$. We call the parameters $(x^-,\allowbreak x^+,\allowbreak y^-,\allowbreak y^+,\allowbreak u^-,\allowbreak u^+,\allowbreak v^-,\allowbreak v^+)$ \emph{normalised} if they cannot be strengthened, i.e.\ if increasing any of the lower bounds or decreasing any of the upper bounds always leads to a different octagon. As the following proposition reveals, we can easily normalise any set of parameters.\footnote{All proofs can be found in the appendix.}

\begin{proposition}
Consider the following set of parameters:
\begin{align*}
x_1^- =&\max(x^-, v^- - y^+, y^- - u^+, 0.5 \cdot (v^- - u^+))\\
x_1^+ =& \min(x^+, v^+ - y^-, y^+ - u^-, 0.5 \cdot (v^+ - u^-))\\
y_1^- =&\max(y^-, u^- + x^-, v^- -x^+, 0.5 \cdot (u^- + v^-) )\\
y_1^+ =&\min(y^+, u^+ + x^+, v^+ -x^-, 0.5 \cdot (u^+ + v^+))\\
u_1^- =&\max(u^-, y^- - x^+,v^- - 2x^+, 2y^- - v^+ )\\
u_1^+ =&\min(u^+, y^+ - x^-,v^+ - 2x^-, 2y^+ - v^- )\\
v_1^- =&\max(v^-, x^- + y^-, u^- + 2x^-, 2y^- - u^+ )\\
v_1^+ =&\min(v^+, x^+ + y^+, u^+ + 2x^+, 2y^+ - u^-)
\end{align*}
Then $\mathsf{Octa}(x^-,x^+,y^-,y^+,u^-,u^+,v^-,v^+)=\mathsf{Octa}(x_1^-,\allowbreak x_1^+,\allowbreak y_1^-,\allowbreak y_1^+,\allowbreak u_1^-,\allowbreak u_1^+,\allowbreak v_1^-,\allowbreak v_1^+)$. Furthermore, if we have $x_1^->x_1^+$, $y_1^->y_1^+$, $u_1^- > u_1^+$ or $v_1^- > v_1^+$ then $\mathsf{Octa}(x^-,\allowbreak x^+,\allowbreak y^-,\allowbreak y^+,\allowbreak u^-,\allowbreak u^+,\allowbreak v^-,\allowbreak v^+)=\emptyset$. Otherwise, the set of parameters $(x_1^-,\allowbreak x_1^+,\allowbreak y_1^-,\allowbreak y_1^+,\allowbreak u_1^-,\allowbreak u_1^+,\allowbreak v_1^-,\allowbreak v_1^+)$ is normalised.
\end{proposition}

\begin{figure}
\centering
\includegraphics[width=220pt]{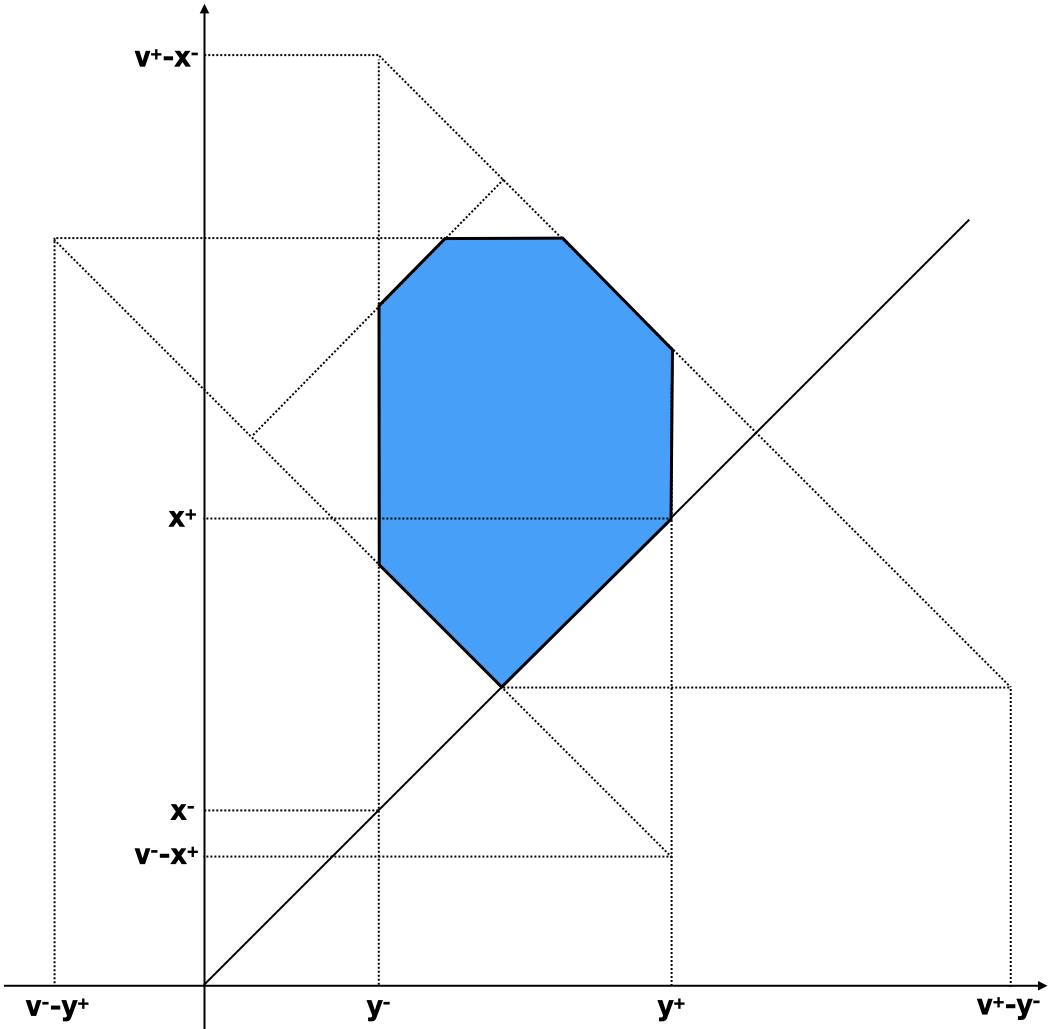}
\caption{Example of an octagon encoding a transitive relation. \label{figOctagonTransitive}}
\end{figure}

\paragraph{Capturing relational properties.}
As we saw in the previous section, two operations are of particular importance for capturing rules: intersection and composition.
The \emph{intersection} of $\mathsf{Octa}(x_1^-,\allowbreak x_1^+,\allowbreak y_1^-,\allowbreak y_1^+,\allowbreak u_1^-,\allowbreak u_1^+,\allowbreak v_1^-,\allowbreak v_1^+)$ and  $\mathsf{Octa}(x_2^-,\allowbreak x_2^+,\allowbreak y_2^-,\allowbreak y_2^+,\allowbreak u_2^-,\allowbreak u_2^+,\allowbreak v_2^-,\allowbreak v_2^+)$ is the following octagon: 
\begin{align*}
\mathsf{Octa}(&\max(x_1^-,x_2^-),\allowbreak \min(x_1^+,x_2^+),\max(y_1^-,y_2^-),\\
&\min(y_1^+,y_2^+),\allowbreak \max(u_1^-,u_2^-),\allowbreak \min(u_1^+u_2^+),\\
&\max(v_1^-,v_2^-),\allowbreak \min(v_1^+,v_2^+))
\end{align*}
Octagons are also closed under \emph{composition}, as the following proposition reveals.
\begin{proposition}
Let $O_1 = \mathsf{Octa}(x_1^-,\allowbreak x_1^+,\allowbreak y_1^-,\allowbreak y_1^+,\allowbreak u_1^-,\allowbreak u_1^+,\allowbreak v_1^-,\allowbreak v_1^+)$ and $O_2 = \mathsf{Octa}(x_2^-,\allowbreak x_2^+,\allowbreak y_2^-,\allowbreak y_2^+,\allowbreak u_2^-,\allowbreak u_2^+,\allowbreak v_2^-,\allowbreak v_2^+)$ be non-empty octagons with normalised parameters. Then $O_1\diamond O_2 = \mathsf{Octa}(x_3^-,\allowbreak x_3^+,\allowbreak y_3^-,\allowbreak y_3^+,\allowbreak u_3^-,\allowbreak u_3^+,\allowbreak v_3^-,\allowbreak v_3^+)$, where:
\begin{align*}
x_3^- &= \max(x_1^-, x_2^- - u_1^+, v_1^- - x_2^+)\\
x_3^+ &= \min(x_1^+, x_2^+ - u_1^-,v_1^+ - x_2^-)\\
y_3^- &= \max(y_2^-, u_2^-+y_1^-, v_2^- - y_1^+)\\
y_3^+ &= \min(y_2^+,u_2^+ + y_1^+,v_2^+-y_1^-)\\
u_3^- &= \max(y_2^- - x_1^+, u_2^- + u_1^-, v_2^- - v_1^+)\\
u_3^+ &= \min(y_2^+ - x_1^-,u_2^+ + u_1^+, v_2^+ - v_1^-)\\
v_3^- &= \max(x_1^- + y_2^-, u_2^-+v_1^-,v_2^- -u_1^+)\\
v_3^+ &= \min(x_1^+ + y_2^+, u_2^+ + v_1^+,v_2^+ - u_1^-)
\end{align*}
\end{proposition}
The \emph{inverse} of a relation can also be straightforwardly characterised. Let us define the inverse of an octagon $O = \mathsf{Octa}(x^-,x^+,y^-,y^+,u^-,u^+,v^-,v^+)$ as: 
\begin{align}\label{eqInvOctagon}
O^{\text{inv}} =\mathsf{Octa}(y^-,y^+,x^-,x^+,-u^+,-u^-,v^-,v^+)
\end{align}
Then clearly we have $(x,y)\in O$ iff $(y,x)\in O^{\text{inv}}$, hence the relation encoded by $O^{\text{inv}}$ is indeed the inverse of the relation encoded by $O$. It also follows that $O$ is \emph{symmetric} if the following conditions are satisfied: $x^-=y^-$, $x^+=y^+$ and $u^- = -u^+$. Similarly, assuming the parameters are normalised, the relation captured by the octagon $O$ is \emph{reflexive} over its domain $[x^-,x^+]$, i.e.\ $\forall x\in[x^-,x^+]\,.\, (x,x)\in O$, iff the following conditions are satisfied:
\begin{itemize}
\item $u^-\leq 0 \leq u^+$;
\item $v^- \leq 2x^- \leq 2x^+ \leq v^+$.
\end{itemize}
From the definition of composition, it straightforwardly follows that a non-empty octagon with normalised coordinates is \emph{transitive}, i.e.\ satisfies $O\diamond O \subseteq O$, iff the following conditions are satisfied:
\begin{itemize}
\item $u^-\leq \max(y^- -x^+, 2u^-, v^- -v^+)$;
\item $u^+\geq \min(y^+ -x^-, 2u^+, v^+ -v^-)$;
\item $v^-\leq \max(x^- + y^-, u^- + v^-, v^- - u^+)$;
\item $v^+\geq \min(x^+ + y^+, u^+ + v^+, v^+ - u^-)$.
\end{itemize}
We can also show the following.
\begin{proposition}\label{propTransitivity}
Let $O=\mathsf{Octa}(x^-,\allowbreak x^+,\allowbreak y^-,\allowbreak y^+,\allowbreak u^-,\allowbreak u^+,\allowbreak v^-,\allowbreak v^+)$ and assume that these parameters are normalised. It holds that $O\diamond O = O$ iff one of the following conditions is satisfied:
\begin{itemize}
\item $O=\emptyset$;
\item $O = [x^-,x^+] \times [y^-,y^+]$ with $x^+ \geq y^-$ and $x^- \leq y^+$;
\item $O = \{(x,x) \,|\, x\in [x^-,x^+]\}$;
\item $u^- =0$, $u^+ >0$, $v^- < v^+$, $v^--x^+\leq x^-$, $x^+\leq v^+-x^-$, $v^--y^+\leq y^-$, $y^+\leq v^+-y^-$ and  $u^+ = \min(y^+ - x^-,v^+ - v^-)$;
\item $u^- < 0$, $u^+=0$, $v^- < v^+$, $v^--x^+\leq x^-$, $x^+\leq v^+-x^-$, $v^--y^+\leq y^-$, $y^+\leq v^+-y^-$ and  $u^- = \max(y^- - x^+,v^- - v^+)$.
\end{itemize}
\end{proposition}
Figure \ref{figOctagonTransitive} shows an octagon which satisfies the conditions from the fourth case. 
If we compose an octagon with itself a sufficient number of times, we always end up with one of the octagon types from Proposition \ref{propTransitivity}. In particular, let us define $O^{(m)}$ as $O^{(m-1)}\diamond O$ for $m\geq 2$ and $O^{(1)}=O$. We then have the following result.
\begin{proposition}
If $v^-<v^+$ then there exists some $m\geq 1$ such that $O^{(m)}\diamond O=O^{(m)}$.
\end{proposition}
\section{Expressivity of Octagon Embeddings}\label{secExpressivity}
We now study the ability of octagon embeddings to capture knowledge graphs and rules. We will denote a given knowledge graph embedding as $\gamma$. Such an embedding represents each entity $e\in \mathcal{E}$ as the vector $\gamma(e)\in\mathbb{R}^n$ and each relation $r\in\mathcal{R}$ as the region $\gamma(r)\subseteq \mathbb{R}^{2n}$. We say that $\gamma$ is an octagon embedding if it is a coordinate-wise model in which $\gamma(r)$ is defined in terms of octagons, for each $r\in\mathcal{R}$.

\subsection{Capturing Knowledge Graphs} 
With each region based embedding $\gamma$ we can associate the knowledge graph $\mathcal{G}_{\gamma}=\{(e,r,f) \,|\, \gamma(e)\oplus\gamma(f)\in \gamma(r)\}$. An important question is whether a given region based embedding model is capable of modelling any knowledge graph, i.e.\ whether for any knowledge graph $\mathcal{G}\subseteq \mathcal{E}\times\mathcal{R}\times\mathcal{E}$ there exists an embedding $\gamma$ such that $\mathcal{G}=\mathcal{G}_{\gamma}$. When using coordinate-wise embeddings with TransE bands, this is not possible \cite{DBLP:conf/aaai/WangGL18,DBLP:conf/nips/Kazemi018}. In fact, for a coordinate-wise model with hexagons of the form \eqref{eqHexagons} this is still not possible as the following proposition reveals.

\begin{proposition}\label{propHexagonLimitation}
Let $X_r$ be a hexagon region of the form \eqref{eqHexagons} and let $\mathbf{e},\mathbf{f}\in \mathbb{R}^n$. If $\mathbf{e}\oplus\mathbf{f}\in X_r$ and $\mathbf{f}\oplus\mathbf{e}\in X_r$ then we also have $\mathbf{e}\oplus\mathbf{e}\in X_r$.
\end{proposition}

From this proposition, it follows that the hexagon model cannot correctly capture a knowledge graph containing the triples $(e,r,f)$ and $(f,r,e)$ but not $(e,r,e$). However, if we use octagons instead of hexagons, we can correctly capture any knowledge graph.

\begin{proposition}\label{propFullyExpressive}
Let $\mathcal{G}\subseteq \mathcal{E}\times\mathcal{R}\times\mathcal{E}$ be a knowledge graph. There exists an octagon embedding $\gamma$ such that $\mathcal{G}=\mathcal{G}_{\gamma}$.
\end{proposition}

\subsection{Capturing Rules} 
We now analyse which sets of rules can be correctly captured by octagon embeddings. Previous work \cite{DBLP:conf/nips/AbboudCLS20} has focused in particular on the following types of rules\footnote{\emph{Asymmetry} rule was called \emph{anti-symmetry} in \cite{DBLP:conf/nips/AbboudCLS20}. We use the former to avoid any confusion with the standard notion of anti-symmetric relations. Furthermore, \cite{DBLP:conf/nips/AbboudCLS20} defined \emph{inversion} as $r_1(X,Y) \equiv r_2(Y,X)$. We consider inversion as a rule, since the equivalence can be straightforwardly recovered by two inversion rules.}:
\begin{align}
&\hspace{-6pt}\textbf{Symmetry:}  && \hspace{-5pt} r(X,Y) \rightarrow r(Y,X) \label{eqPatternSymmetry}\\
&\hspace{-6pt}\textbf{Asymmetry:} &&\hspace{-5pt} r(X,Y) \rightarrow \neg r(Y,X) \label{eqPatternAntiSymmetry}\\
&\hspace{-6pt}\textbf{Inversion:} && \hspace{-5pt}r_1(X,Y) \rightarrow r_2(Y,X) \label{eqPatternInversion}\\
&\hspace{-6pt}\textbf{Hierarchy:} && \hspace{-5pt}r_1(X,Y) \rightarrow r_2(X,Z) \label{eqPatternHierarchy}\\
&\hspace{-6pt}\textbf{Intersection:} &&\hspace{-5pt} r_1(X,Y) \wedge r_2(X,Y) \rightarrow r_3(X,Y) \label{eqPatternIntersection}\\
&\hspace{-6pt}\textbf{Mutual exclusion:} && \hspace{-5pt} r_1(X,Y) \wedge r_2(X,Y) \rightarrow \bot \label{eqPatternMutualExclusion}\\
&\hspace{-6pt}\textbf{Composition:} && \hspace{-5pt}r_1(X,Y) \wedge r_2(Y,Z) \rightarrow r_3(X,Z) \label{eqPatternComposition}
\end{align}
Let us define the inverse region $X_r^{\textit{inv}}$ as follows:
$$
X_r^{\textit{inv}}{=}\{(x_1,...,x_n,y_1,...,y_n) \,|\, (y_1,...,y_n,x_1,...,x_n)\in X_r\}
$$
Note that when $X_r=[O_1,\ldots,O_n]$ then $X_r^{\textit{inv}}=[O_1^{\textit{inv}},\allowbreak \ldots,\allowbreak O_n^{\textit{inv}}]$ with $O_i^{\textit{inv}}$ as defined in \eqref{eqInvOctagon}.
Each of the aforementioned rule types can be straightforwardly modelled using regions. In particular, we say that a region based embedding $\gamma$ captures or satisfies \eqref{eqPatternSymmetry} iff $\gamma(r)\subseteq \gamma(r)^{\textit{inv}}$; \eqref{eqPatternAntiSymmetry} iff  $\gamma(r) \cap \gamma(r)^{\textit{inv}}=\emptyset$; \eqref{eqPatternInversion} iff $\gamma(r_1)\subseteq\gamma(r_2)^{\textit{inv}}$; \eqref{eqPatternHierarchy} iff $\gamma(r_1)\subseteq \gamma(r_2)$; \eqref{eqPatternIntersection} iff $\gamma(r_1)\cap \gamma(r_2)\subseteq \gamma(r_3)$; \eqref{eqPatternMutualExclusion} iff $\gamma(r_1)\cap\gamma(r_2)=\emptyset$; and \eqref{eqPatternComposition} iff $\gamma(r_1)\diamond \gamma(r_2)\subseteq \gamma(r_3)$. For a set of rules $\mathcal{K}$ and a rule $\rho$, we write $\mathcal{K}\models \rho$ to denote that $\rho$ is entailed by $\mathcal{K}$ in the usual sense. 
We also write $\gamma\models \rho$ to denote that the embedding $\gamma$ captures the rule $\rho$. 

We are interested in analysing whether a given region based embedding model can capture a set of rules $\mathcal{K}$, without capturing any other rules, i.e.\ such that for any rule $\rho$ it holds that $\gamma\models \rho$ iff $\mathcal{K}\models \rho$. Whether this is possible in general depends on the kinds of rules which $\mathcal{K}$ is permitted to contain. 
We first show a negative result. We call a region based embedding $\gamma$ convex if $\gamma(r)$ is convex for every $r\in \mathcal{R}$. 
\begin{proposition}\label{propNoMutualExclusion}
There exists a consistent set of hierarchy, intersection and mutual exclusion rules $\mathcal{K}$ such that every convex embedding of the form \eqref{eqCoordinateWise} or \eqref{eqBoxE} which satisfies $\mathcal{K}$ also satisfies some hierarchy rule which is not entailed by $\mathcal{K}$. 
\end{proposition}

We can show the same result when only hierarchy, intersection and asymmetry rules are permitted. Note that Proposition \ref{propNoMutualExclusion} also applies to BoxE, and thus contradicts the claim from \cite{DBLP:conf/nips/AbboudCLS20} that BoxE can capture arbitrary consistent sets of symmetry, asymmetry, inversion, hierarchy, intersection and mutual exclusion rules. The underlying issue relates to the fact that the considered embeddings are defined from two-dimensional regions, and when these regions are convex, Helly's theorem imposes restrictions on when intersections of such regions can be disjoint. Furthermore, note that \cite{DBLP:conf/nips/AbboudCLS20} did not consider composition rules, as such rules cannot be captured by BoxE. When we omit asymmetry, mutual exclusion and composition rules, we can capture any rule base, as the following result shows. 

\begin{proposition}
Let $\mathcal{K}$ be a set of symmetry, inversion, hierarchy and intersection rules. There exists a coordinate-wise octagon embedding $\gamma$ which satisfies $\mathcal{K}$, and which only satisfies those symmetry, inversion, hierarchy and intersection rules which are entailed by $\mathcal{K}$, and which does not satisfy any asymmetry, mutual exclusion and composition rules.
\end{proposition}

Let us now shift our focus to composition rules. First, we have the following negative result, showing that octagon embeddings cannot capture arbitrary sets of composition rules.

\begin{proposition}\label{propNoArbitraryComposition}
Let $\mathcal{K} = \{r_1(X,Y) \wedge r_1(Y,Z) \rightarrow r_2(X,Z), r_2(X,Y) \wedge r_2(Y,Z) \rightarrow r_3(X,Z)\}$. It holds that every octagon embedding $\gamma$ which satisfies $\mathcal{K}$ also satisfies some composition rule which is not entailed by $\mathcal{K}$.
\end{proposition}
The underlying issue relates to composition rules where the same relation appears more than once. If we exclude such cases, we can correctly embed sets of composition rules. 

\begin{definition}
An \emph{extended composition rule} over a set of relations $\mathcal{R}$ is an expression of the form:
$$
r_1(X_1,X_2)\wedge \ldots \wedge r_k(X_k,X_{k+1}) \rightarrow r_{k+1}(X_1,X_{k+1})
$$
where $k\geq 1$ and $r_1,\ldots,r_{k+1}\in \mathcal{R}$. Such a rule is called \emph{regular} if we have $r_i\neq r_j$ for $i\neq j$. 
\end{definition}
Regular composition rules generalise the hierarchy and composition rules from \cite{DBLP:conf/nips/AbboudCLS20}, noting that the latter work also required the three relations appearing in composition rules to be distinct.

\begin{proposition}
Let $\mathcal{K}$ be a set of regular composition rules. Assume that any extended composition rule entailed by $\mathcal{K}$ is either a trivial rule of the form $r\subseteq r$ or a regular rule. There exists an octagon embedding $\gamma$ which satisfies $\mathcal{K}$, and which only satisfies those extended composition rules which are entailed by $\mathcal{K}$.
\end{proposition}

This result is significant given the importance of composition rules for link prediction \cite{DBLP:conf/ijcai/MeilickeCRS19}, among others. Moreover, a similar result was not yet shown in previous work on region based embeddings, to the best of our knowledge. At the same time, however, the result is limited by the fact that only regular rules are considered and no inverse relations. Knowledge graph completion often relies (explicitly or implicitly) on rules such as $\textit{playsForTeam}(X,Y)\wedge \textit{playsForTeam}^{-1}(Y,Z) \wedge \textit{playsSport}(Z,U) \rightarrow \textit{playsSport}(X,U)$. Our analysis suggests that capturing sets of such rules would require an octagon model with cross-coordinate comparisons, which we leave as a topic for future work.

\section{Learning Octagon Embeddings}\label{secLearningOctagonEmbeddings}
In our theoretical analysis, determining whether a triple $(e,r,f)$ is supported by an embedding was based on a discrete criterion: either $\mathbf{e} \oplus \mathbf{f}$ lies in $X_r$ or it does not. In practice, however, we need to learn regions with fuzzy boundaries, given that link prediction is typically treated as a ranking problem, rather than a classification problem. Continuous representations are typically also easier to learn than discrete structures. Let us consider a region $X_r=[O_1^r,\ldots,O_n^r]$, where $O_i^r = \mathsf{Octa}(x_i^-,\allowbreak x_i^+,\allowbreak,y_i^-,\allowbreak y_i^+,\allowbreak,u_i^-,\allowbreak u_i^+,\allowbreak,v_i^-,\allowbreak v_i^+,\allowbreak)$. Let us write $\mathbf{x}^- = (x_1^-,\ldots,x_n^-)$ and similar for $\mathbf{x}^+$, $\mathbf{y}^-$, $\mathbf{y}^+$, $\mathbf{u}^-$, $\mathbf{u}^+$, $\mathbf{v}^-$ and $\mathbf{v}^+$. Octagon embeddings impose the following four constraints on an entity pair $(e,f)$:
\begin{align*}
\mathbf{x}^- &\leq \mathbf{e} \leq  \mathbf{x}^+ &
\mathbf{y}^- &\leq \mathbf{f} \leq  \mathbf{y}^+ \\
\mathbf{u}^- &\leq \mathbf{f}-\mathbf{e} \leq  \mathbf{u}^+ &
\mathbf{v}^- &\leq \mathbf{e}+\mathbf{f} \leq  \mathbf{v}^+
\end{align*}
We will refer to the regions defined by these four constraints as \emph{bands}.
We can straightforwardly use the sigmoid function $\sigma$ to convert the constraints into soft scores. For the $u$-constraint we can use the following ``distance'' to the $u$-band:
\begin{align}\label{eqScoreUBand}
\text{dist}_u(e,r,f) &= \sigma\big(|(\mathbf{f} - \mathbf{e}) - \mathbf{u}_c| - \mathbf{u}_w\big)
\end{align}
where we write $\mathbf{u}_c = (\mathbf{u}^+ + \mathbf{u}^-)/2$ and $\mathbf{u}_w = (\mathbf{u}^+ - \mathbf{u}^-)/2$. 
The functions $\text{dist}_x$, $\text{dist}_y$ and $\text{dist}_v$ are defined analogously. Note that $\mathbf{e}\oplus\mathbf{f}\in X_r$ iff $\text{dist}_x(e,r,f)\geq 0.5$, $\text{dist}_y(e,r,f)\geq 0.5$, $\text{dist}_u(e,r,f)\geq 0.5$ and $\text{dist}_v(e,r,f)\geq 0.5$.
Furthermore note that BoxE and ExpressivE also convert interval constraints into soft scores. However, rather than using scores of the form \eqref{eqScoreUBand}, these models use a piecewise linear function of the distance to the centre of the interval. Although a different slope is used for points inside and outside the interval, we may question to what extent the resulting embeddings still define regions. While the regions defined by scoring functions of the form \eqref{eqScoreUBand} also have fuzzy boundaries, this ``fuzziness'' is mostly limited to the immediate vicinity of the boundary. 

We also consider a variant of the distance function \eqref{eqScoreUBand} which uses learnable attention weights $\mathbf{u}_\text{att}$, as follows:
\begin{equation}\label{eqWeightedDistanceU}
    \text{dist}'_u(e,r,f) = \mathbf{u}_\text{att} \odot \text{dist}_u(e,r,f)
\end{equation}
where we write $\odot$ for component-wise multiplication. Weighted variants of $\text{dist}_x$, $\text{dist}_y$ and $\text{dist}_v$ are defined analogously. The attention weights intuitively allow the model to ``forget'' certain constraints, e.g.\ by focusing only (or primarily) on the $u$ band in some coordinates.

As in BoxE and ExpressivE, coordinate-wise scores are aggregated to a scalar by taking their negated norm. We define $\text{dist}(e,r,f) = \text{dist}_x(e,r,f) \oplus \text{dist}_y(e,r,f) \oplus \text{dist}_u(e,r,f) \oplus \text{dist}_v(e,r,f)$ and score triples as follows (and similar for the weighted variant):
\begin{equation}\label{eqLogisticScore}
    s(e,r,f) = -\|\:\text{dist}(e,r,f)\:\|_p
\end{equation}
In the above expression, $p$ is a hyperparameter.
Learning occurs by minimizing the following margin loss function with self-adversarial negative sampling~\cite{DBLP:conf/iclr/SunDNT19}:
\begin{align*}
-\log\big(\sigma(\lambda - s(e,r,f))\big) - \sum_{i = 1}^{|N|} \alpha_i \log\big(\sigma(s(e_i,r,f_i) - \lambda)\big)
\end{align*}
where $\mathbf{\alpha} = \text{softmax}(s(e_1,r,f_1) \oplus \ldots \oplus s(e_n,r,f_n)$, the margin $\lambda$ is a hyperparameter, and $N = \{ (e_1,r,f_1),\allowbreak \ldots ,\allowbreak (e_n,r,f_n) \}$ is a set of randomly sampled negative examples.



\section{Experimental Results}
We consider several variants of our model. The full model is denoted by \textit{uvxy}, whereas \textit{ux} refers to a model in which $\text{dist}_v$ and $\text{dist}_y$ are not used, and similar for the other variants. For $uv$ and $uvxy$, we have used the variant with attention weights \eqref{eqWeightedDistanceU}. For the other configurations, attention weights were not used, as they were found to make little difference\footnote{An analysis of the impact of attention weights can be found in the appendix (Section~\ref{sec:ExpResults}).}. Note that our $u$ model is almost identical to TransE, except that an explicit width parameter is used, while \textit{uv} can be seen as a variant of ExpressivE without learnable scale factors. 

Table~\ref{tab:link-prediction} compares the performance of our model with TransE~\cite{DBLP:conf/nips/BordesUGWY13}, BoxE~\cite{DBLP:conf/nips/AbboudCLS20} and ExpressivE~\cite{DBLP:conf/iclr/0002S23} for link prediction, using the two most common benchmarks~\cite{DBLP:conf/acl-cvsc/ToutanovaC15,DBLP:conf/aaai/DettmersMS018}. 
As usual, performance is measured using Hits@$k$ and Mean Reciprocal Rank (MRR).
All results are for 1000-dimensional embeddings for FB15k-237 and 500-dimensional embeddings for WN18RR, following BoxE and ExpressivE's experimental setup. Other hyperparameters are specified in the appendix.

\begin{table}[]
    \centering
    \footnotesize
    \setlength\tabcolsep{2pt}
    \begin{tabular}{lcccccccc}
        \toprule
        & \multicolumn{4}{c}{FB15k-237} & \multicolumn{4}{c}{WN18RR} \\
        \cmidrule(lr){2-5}\cmidrule(lr){6-9}
        & H@1 & H@3 & H@10 & MRR & H@1 & H@3 & H@10 & MRR \\
        \midrule
        TransE & 22.3 & 37.2 & 53.1 & 33.2 & 01.3 & 40.1 & 52.9 & 22.3 \\
        BoxE & 23.8 & \textbf{37.4} & \textbf{53.8} & \textbf{33.7} & 40.0 & 47.2 & 54.1 & 45.1 \\
        ExpressivE & \textbf{24.3} & 36.6 & 51.2 & 33.3 & \textbf{46.4} & \textbf{52.2} & \textbf{59.7} & \textbf{50.8} \\
        \midrule
        
        $u$ & 23.1 & \underline{37.3} & \underline{53.2} & 33.2 & 01.6 & 39.9 & 51.5 & 22.0 \\
        $ux$ & 23.3 & 37.1 & 52.5 & 33.1 & 01.9 & 39.0 & 51.6 & 21.8 \\
        $uxy$ & 23.2 & 37.2 & 53.1 & 33.2 & 01.8 & 40.8 & 52.8 & 22.8 \\
        $uv^*$ & \underline{24.1} & 36.9 & 52.8 & \underline{33.6} & 43.6 & 48.5 & 52.9 & 46.9 \\
        $uvxy^*$ & \underline{24.1} & 36.7 & 51.7 & 33.2 & \underline{43.6} & \underline{49.2} & \underline{56.1} & \underline{47.9} \\
        \bottomrule
    \end{tabular}
    \caption{Link prediction performances of region based embedding models. Configurations with $^*$ use the variant with attention weights.}
    \label{tab:link-prediction}
\end{table}

\begin{figure}
    \centering
    \includegraphics[width=0.49\columnwidth]{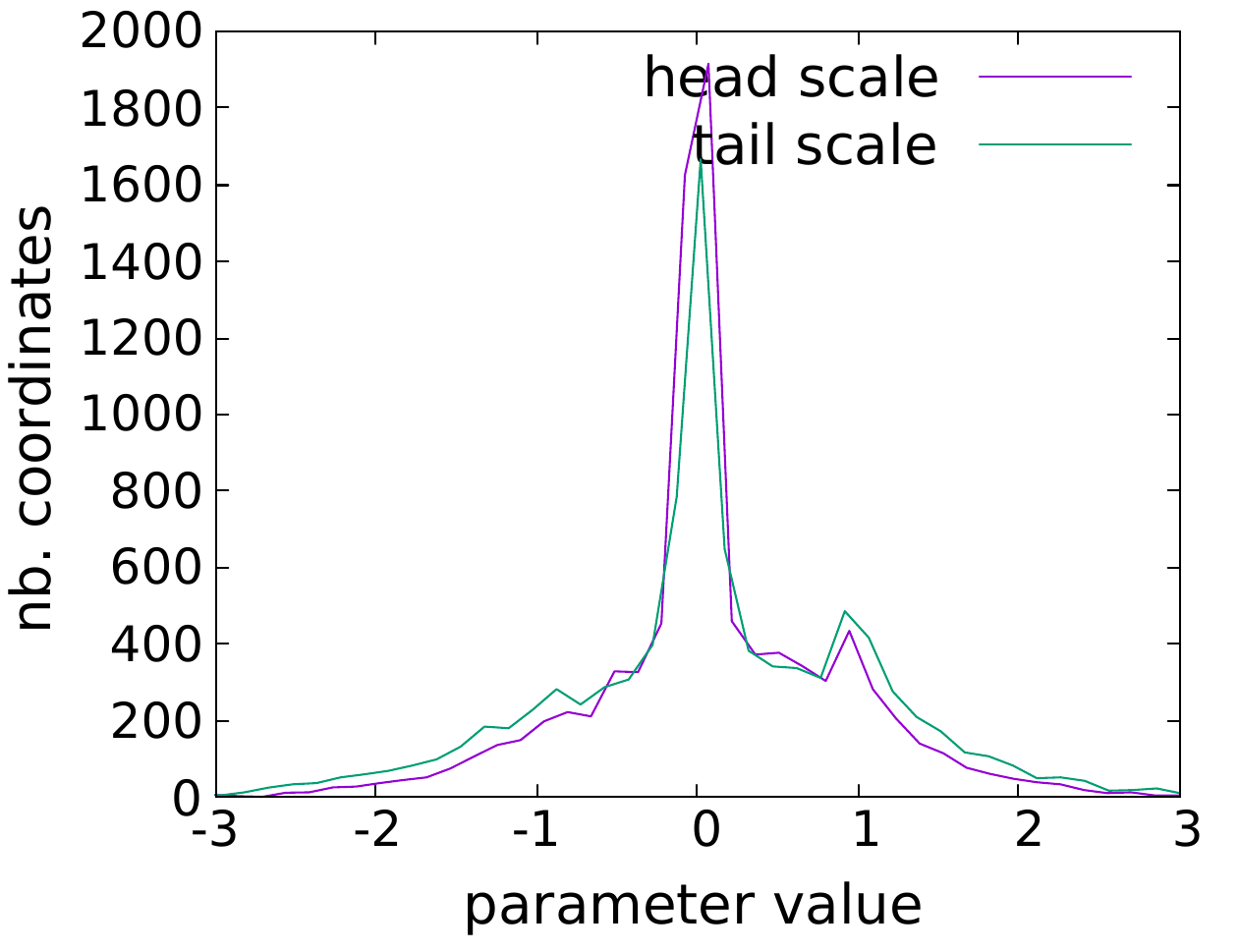}
    \includegraphics[width=0.49\columnwidth]{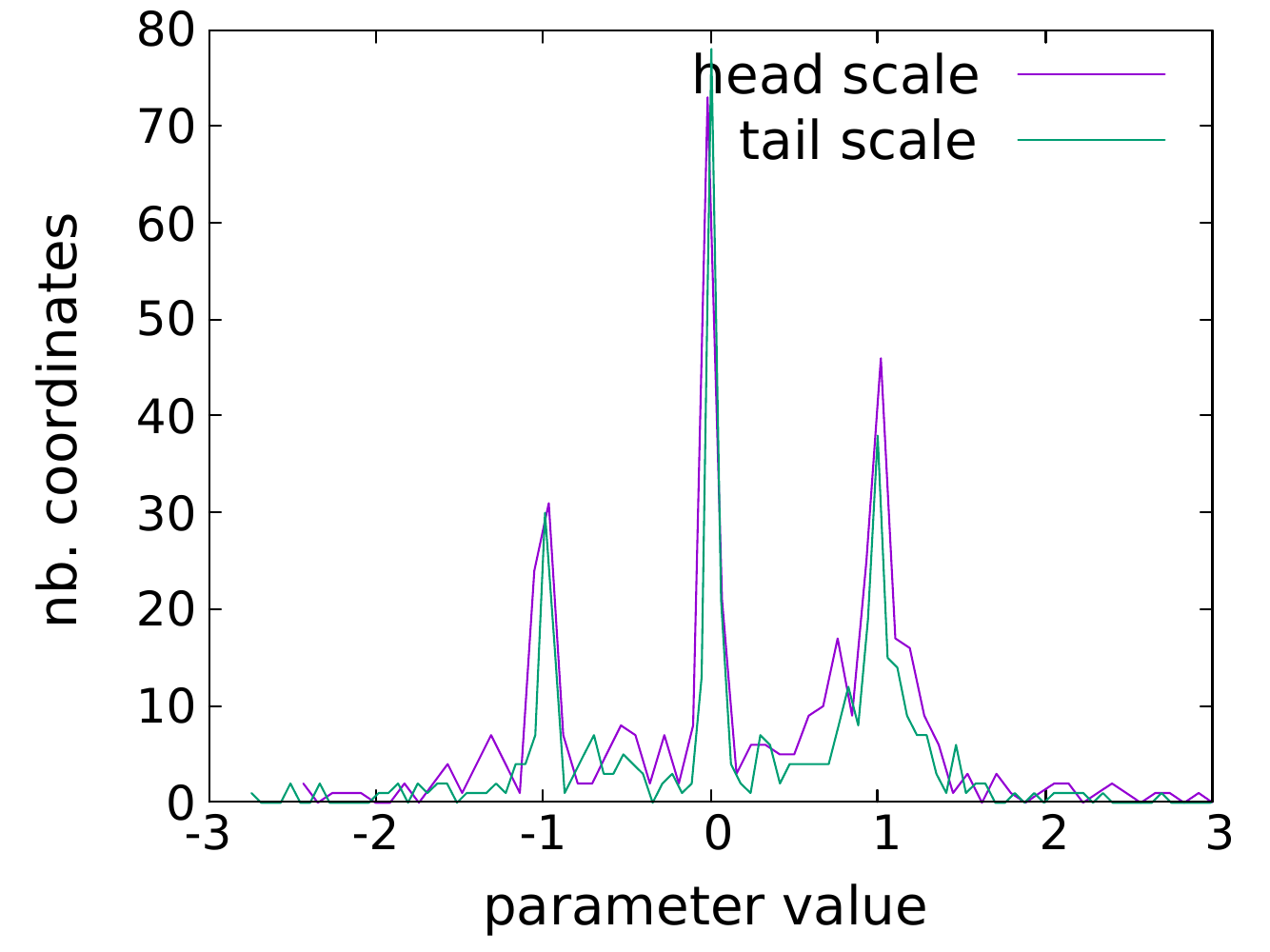}
    \caption{Distribution of scale factors in 40-dimensional ExpressivE embeddings for FB15k-237 (left) and WN18RR (right).}
    \label{fig:expressive-scale}
\end{figure}


The performance of $u$ is almost identical to TransE, which suggests that the addition of an explicit width parameter plays no significant role. Adding the $x$, $y$ and $v$ constraints does not bring any benefits on FB15k-237.
On WN18RR, however, introducing the $v$ constraint has a big impact on H@1, which goes from less than 1\% for $u$, $ux$ and $uxy$ to more than 40\% for $uv$ and $uvxy$. This result is in line with Proposition~\ref{propHexagonLimitation}, which states that hexagons cannot capture relations that are both symmetric and irreflexive: for most WN18RR relations, models without $v$ constraints incorrectly rank $(e,r,e)$ triples first\footnote{This result still holds without attention weights: on WN18RR, $uvxy$ achieves a H@1 of 39.6, H@3 of 45.9, H@10 of 50.7 and MRR of 43.8 without trainable attention weights.}. Including the $x$ and $y$ constraints also improves performances in the case of WN18RR.
Introducing attention weights, as in \eqref{eqWeightedDistanceU}, has a significant positive effect on the overall performance on WN18RR but not for FB15k-237. As Table~\ref{tab:link-prediction} shows, $uvxy$ outperforms BoxE on all metrics. Further analysis can be found in the appendix.

Octagon embeddings slightly underperform ExpressivE in most cases. We may wonder whether this reflects the inherent limitations of using fixed angles. In fact, when inspecting the scale factors in learned ExpressivE embedddings, we noticed that most are close to $-1$, $0$ or $1$ (see Figure~\ref{fig:expressive-scale}). In other words, most of the learned parallelograms are diamonds ($uv$), bands ($u$, $v$, $x$, $y$) or squares ($xy$), all of which can be represented in our octagon model. This suggests that the underperformance of the octagon model is rather  the result of overparameterisation. On the one hand, we have shown theoretically that all four constraints are needed to capture semantic dependencies, and this is also confirmed by our empirical results on WN18RR. On the other hand, not all constraints might be needed for all coordinates. The use of attention weights was inspired by this view, but the weights themselves of course also introduce further parameters.


\section{Related Work}


Knowledge graph embeddings have generally been classified in three broad categories: linear models, tensor decomposition models and neural models~\cite{DBLP:series/synthesis/2021Hogan}.
Linear models use a distance to score triples, after applying a linear transformation to entity or relation embeddings, e.g.\ a translation~\cite{DBLP:conf/nips/BordesUGWY13} or rotation~\cite{DBLP:conf/iclr/SunDNT19}. 
Region based models can essentially be viewed under the same umbrella, where the idea of computing distances between points is generalised to evaluating relationships between points and regions.
Tensor decomposition models, such as ComplEx~\cite{DBLP:journals/jmlr/TrouillonDGWRB17,DBLP:conf/icml/LacroixUO18}, TuckER~\cite{DBLP:conf/emnlp/BalazevicAH19} and QuatE~\cite{DBLP:conf/nips/0007TYL19}, view knowledge graphs as a three-dimensional tensor, which can be factorised into (low-rank) entity and relation embeddings. While these models tend to perform well, they often come with a significantly higher computational cost. Another important limitation is that they are far less interpretable in terms of how they capture semantic dependencies between relations. Tensor decomposition models also have important theoretical limitations when it comes to capturing rules~\cite{DBLP:conf/kr/Gutierrez-Basulto18}.
Neural models, 
including ConvE~\cite{DBLP:conf/aaai/DettmersMS018} and graph neural network based approaches \cite{DBLP:conf/nips/ZhuZXT21}, are even less interpretable than tensor decomposition models.



Knowledge graph embeddings have been combined with rules in various ways. For instance, it has been proposed to incorporate learned rules \cite{DBLP:conf/aaai/GuoWWWG18} or hard ontological constraints \cite{DBLP:conf/nips/AbboudCLS20} when learning knowledge graph embeddings. Conversely, learned embeddings have been exploited for implementing an efficient rule mining system \cite{DBLP:conf/ijcai/OmranWW18}. The ability of models to capture rules has also been evaluated indirectly, for instance by analysing their inductive generalisation capabilities \cite{DBLP:journals/jair/TrouillonGDB19} or their ability to capture intersections and compositions of relations for evaluating complex queries \cite{DBLP:conf/iclr/ArakelyanDMC21,DBLP:conf/aaai/WangCG23}.

Region based approaches are closely related to ontology embeddings. In particular, several models have recently been introduced for the Description Logic $\mathcal{EL}^{++}$, which features subsumption and composition of relations, as well as union and intersection of concepts. Concepts are typically embedded as hyperballs~\cite{DBLP:conf/aaaiss/MondalBM21} or hypercubes~\cite{DBLP:conf/semweb/XiongPTNS22}. The Box$^2$EL model~\cite{DBLP:journals/corr/abs-2301-11118} also represents relations as regions, based on BoxE, which was found to improve performance. Given the limitations of BoxE on modeling composition, an evaluation of more expressive region based models is likely to yield even better results, although
such an evaluation is yet to be performed.

\section{Conclusions}
The study of region based knowledge graph embedding essentially aims to identify the simplest models that are capable of capturing particular types of reasoning patterns. Despite the extensive work on knowledge graph embedding, there are still surprisingly many open questions on this front.
Octagons where shown to be capable of capturing arbitrary knowledge graphs, while a simpler model based on hexagons is not.
Octagons also have the desired property of being closed under intersection and composition, which is not the case for closely related models such as BoxE and ExpressivE.
We have furthermore shown that the octagon model is capable of modelling (particular) sets of composition rules, which is an important property that was not yet shown for existing region based models.
Our empirical results show that octagon embeddings slightly underperform ExpressivE 
and further work is needed to better understand whether 
this performance gap can be closed by changing how the octagons are learned.

\appendix

\section{Modelling Relations with Octagons}

\subsection{Parameterisation}
In general, we can strengthen the bounds of an octagon $\mathsf{Octa}(x^-,x^+,y^-,y^+,u^-,u^+,v^-,v^+)$ as follows:
\begin{align}
x^- &\leftarrow \max\left(x^-, v^- - y^+, y^- - u^+, \frac{v^- - u^+}{2}\right) \label{eqUpdateOctaParams1}\\
x^+ &\leftarrow \min\left(x^+, v^+ - y^-, y^+ - u^-, \frac{v^+ - u^-}{2}\right)\\
y^- &\leftarrow \max\left(y^-, u^- + x^-, v^- -x^+, \frac{u^- + v^-}{2} \right)\\
y^+ &\leftarrow \min\left(y^+, u^+ + x^+, v^+ -x^-, \frac{u^+ + v^+}{2} \right)\\
u^- &\leftarrow \max\left(u^-, y^- - x^+,v^- - 2x^+, 2y^- - v^+ \right)\\
u^+ &\leftarrow \min\left(u^+, y^+ - x^-,v^+ - 2x^-, 2y^+ - v^- \right)\\
v^- &\leftarrow \max\left(v^-, x^- + y^-, u^- + 2x^-, 2y^- - u^+ \right)\\
v^+ &\leftarrow \min\left(v^+, x^+ + y^+, u^+ + 2x^+, 2y^+ - u^- \right)\label{eqUpdateOctaParams8}
\end{align}
To see why these update rules are sound, let us consider for instance the first rule. The reason why we can strengthen $x^-$ to $v^-y^+$, in case $x^+< v^- - y^+$ is because $v^-$ is a lower bound on $x+y$ and $y^+$ is an upper bound on $y$, hence $v^- - y^+$ is a lower bound on $(x+y)-y = x$. The other cases are all analogous. It turns out that applying the update rules \eqref{eqUpdateOctaParams1}--\eqref{eqUpdateOctaParams8} once is sufficient for obtaining a normalised set of parameters, as captured by the following proposition.

\begin{proposition}\label{propNormalisation}
Consider the following set of parameters:
\begin{align*}
x_1^- =&\max\left(x^-, v^- - y^+, y^- - u^+, \frac{v^- - u^+}{2}\right)\\
x_1^+ =& \min\left(x^+, v^+ - y^-, y^+ - u^-, \frac{v^+ - u^-}{2}\right)\\
y_1^- =&\max\left(y^-, u^- + x^-, v^- -x^+, \frac{u^- + v^-}{2} \right)\\
y_1^+ =&\min\left(y^+, u^+ + x^+, v^+ -x^-, \frac{u^+ + v^+}{2} \right)\\
u_1^- =&\max\left(u^-, y^- - x^+,v^- - 2x^+, 2y^- - v^+ \right)\\
u_1^+ =&\min\left(u^+, y^+ - x^-,v^+ - 2x^-, 2y^+ - v^- \right)\\
v_1^- =&\max\left(v^-, x^- + y^-, u^- + 2x^-, 2y^- - u^+ \right)\\
v_1^+ =&\min\left(v^+, x^+ + y^+, u^+ + 2x^+, 2y^+ - u^- \right)
\end{align*}
Then $\mathsf{Octa}(x^-,x^+,y^-,y^+,u^-,u^+,v^-,v^+)=\mathsf{Octa}(x_1^-,\allowbreak x_1^+,\allowbreak y_1^-,\allowbreak y_1^+,\allowbreak u_1^-,\allowbreak u_1^+,\allowbreak v_1^-,\allowbreak v_1^+)$. Furthermore, if we have $x_1^->x_1^+$, $y_1^->y_1^+$, $u_1^- > u_1^+$ or $v_1^- > v_1^+$ then $\mathsf{Octa}(x^-,\allowbreak x^+,\allowbreak y^-,\allowbreak y^+,\allowbreak u^-,\allowbreak u^+,\allowbreak v^-,\allowbreak v^+)=\emptyset$. Otherwise, the set of parameters $(x_1^-,\allowbreak x_1^+,\allowbreak y_1^-,\allowbreak y_1^+,\allowbreak u_1^-,\allowbreak u_1^+,\allowbreak v_1^-,\allowbreak v_1^+)$ is normalised.
\end{proposition}
\begin{proof}
The fact that $\mathsf{Octa}(x^-,x^+,y^-,y^+,u^-,u^+,v^-,v^+)=\emptyset$ as soon as $x_1^- > x_1^+$, $y_1^- > y_1^+$, $u_1^- > u_1^+$ or $v_1^- > v_1^+$ follows immediately from the soundness of the update rules \eqref{eqUpdateOctaParams1}--\eqref{eqUpdateOctaParams8}. Let us therefore assume $x_1^- \leq x_1^+$, $y_1^- \leq y_1^+$, $u_1^- \leq u_1^+$ or $v_1^- \leq v_1^+$. Note that this also entails $x^-\leq x^+, y^-\leq y^+, u^-\leq u^+, v^-\leq v^+$.
We show that the following points belong to $\mathsf{Octa}(x^-,x^+,y^-,y^+,u^-,u^+,v^-,v^+)$:
\begin{align*}
&(x_1^-,v_1^- - x_1^-)\\
&(x_1^-,u_1^+ + x_1^-)\\
&(y_1^+ - u_1^+,y_1^+)\\
&(v_1^+ - y_1^+,y_1^+)\\
&(x_1^+,v_1^+-x_1^+)\\
&(x_1^+,u_1^- + x_1^+)\\
&(y_1^- - u_1^-,y_1^-)\\
&(v_1^- - y_1^-,y_1^-)
\end{align*}
This is sufficient to show the proposition, because each of the lower and upper bounds is witnessed for one of these points. For instance, the fact that $(x_1^-,v_1^- - x_1^-)$ belongs to the octagon implies that the lower bound $x_1^-$ is tight, while the fact that $(y_1^+ - u_1^+,y_1^+)$ belongs to the octagon implies that the upper bound $u_1^+$ is tight.

To show that $(x_1^-,v_1^- - x_1^-)$ belongs to the octagon, we have to verify that the following inequalities are satisfied:
\begin{align*}
&x^- \leq x_1^- \leq x^+\\
&y^- \leq v_1^- - x_1^- \leq y^+\\
&u^- \leq v_1^- - 2 x_1^- \leq u^+\\
&v^- \leq v_1^- \leq v^+
\end{align*}
We find:
\begin{itemize}
\item We trivially have $x^- \leq x_1^-$.
\item We obtain $x_1^-\leq x^+$ from the assumption that $x_1^-\leq x_1^+$ and the fact that $x_1^+\leq x^+$.
\item  We now show $y^- \leq v_1^- - x_1^-$:
\begin{itemize}
\item If $x_1^-=x^-$ we have $v_1^- - x_1^- = v_1^- - x^- \geq x^-+y^- - x^- = y^-$.
\item If $x_1= v^- - y^+$ we have $v_1^- - x_1^- \geq v^- - (v^- - y^+) = y^+ \geq y^-$.
\item If $x_1 = y^--u^+$ we have $v_1^- - x_1^- \geq (2y^--u^+)- (y^--u^+) = y^-$.
\item If $x_1 = \frac{v^--u^+}{2}$ we know that $\frac{v^--u^+}{2}\geq y^--u^+$ and thus $\frac{v^-+u^+}{2}\geq y^-$. Using this inequality, we find $v_1^- - x_1^- \geq v^- - \frac{v^--u^+}{2} = \frac{v^-+u^+}{2}\geq y^-$.
\end{itemize}
\item We show $v_1^- - x_1^- \leq y^+$:
\begin{itemize}
\item If $v_1^-=v^-$ we have $v_1^- - x_1^- \leq v^- - (v^--y^+)\leq y^+$.
\item If $v_1^-=x^-+y^-$ we have  $v_1^- - x_1^- \leq x^-+y^- - x^- = y^- \leq y^+$
\item If $v_1^-=u^-+2x^-$ we have $v_1^- - x_1^- \leq u^-+2x^--x^- \leq u^-+x^-$, where $u^-+x^-\leq y^+$ follows from the assumption that $y_1-\leq y_1^+$.
\item If $v_1^- = 2y^--u^+$ we have $v_1^- - x_1^- \leq 2y^--u^+ - (y^--u^+)=y^- \leq y^+$.
\end{itemize}
\item We show $v_1^- - 2x_1^- \geq u^-$:
\begin{itemize}
\item If $x_1^-=x^-$ we have $v_1^- - 2x_1^- \geq (u^- + 2x^-) - 2x^- = u^-$. 
\item If $x_1= v^- - y^+$ we have $v_1^- - 2x_1^- \geq v^- - 2(v^- -y^+) = 2y^+-v^- \geq u_1^+ \geq u_1^-\geq u^-$.
\item If $x_1 = y^--u^+$ we have $v_1^- - 2x_1^- \geq 2y^- - u^+ - 2(y^- - u^+) = u^+ \geq u^-$. 
\item If $x_1 = \frac{v^--u^+}{2}$ we have $v_1^- - 2x_1^- \geq v^- - (v^--u^+) = u^+ \geq u^-$.
\end{itemize}
\item We show $v_1^- - 2x_1^- \leq u^+$:
\begin{itemize}
\item If $v_1^-=v^-$ we have $v_1^- - 2x_1^- \leq v^- 2(\frac{v^--u^+}{2}) = u^+$.
\item If $v_1^-=x^-+y^-$ we know that $x^-+y^- \geq 2y^- -u^+$ and thus $u^+\geq y^--x^-$. Using this inequality we find $v_1^- - 2x_1^- \leq =x^-+y^- -2x^- = y^--x^- \leq u^+$.
\item If $v_1^-=u^-+2x^-$ we have $v_1^- - 2x_1^- \leq u^-+2x^- - 2x^- = u^- \leq u^+$.
\item If $v_1^- = 2y^--u^+$ we have $v_1^- - 2x_1^- \leq 2y^--u^+ - 2(y^- -u^+) = u^+$.
\end{itemize}
\item We trivially have $v^-\leq v_1^-$.
\item We have $v_1^-\leq v+$ because of the assumption that $v_1^-\leq v_1^+$.
\end{itemize}
The result for the other seven points follows by symmetry.
\end{proof}

\subsection{Composing Octagons}
\begin{lemma}\label{lemmaVerticesOctagon}
Let $O=\mathsf{Octa}(x^-,x^+,y^-,y^+,u^-,u^+,v^-,v^+)$ be an octagon with normalised parameters. The vertices of this octagon are as follows:
\begin{align}
&(x^-,v^- - x^-)\label{eqCoordVertices1}\\
&(x^-,u^+ + x^-)\\
&(y^+ - u^+,y^+)\\
&(v^+ - y^+,y^+)\\
&(x^+,v^+-x^+)\\
&(x^+,u^- + x^+)\\
&(y^- - u^-,y^-)\\
&(v^- - y^-,y^-)\label{eqCoordVertices8}
\end{align}
\end{lemma}
\begin{proof}
Trivial.
\end{proof}

\begin{proposition}\label{propCharacterisationComposition}
Let $O_1 = \mathsf{Octa}(x_1^-,\allowbreak x_1^+,\allowbreak y_1^-,\allowbreak y_1^+,\allowbreak u_1^-,\allowbreak u_1^+,\allowbreak v_1^-,\allowbreak v_1^+)$ and $O_2 = \mathsf{Octa}(x_2^-,\allowbreak x_2^+,\allowbreak y_2^-,\allowbreak y_2^+,\allowbreak u_2^-,\allowbreak u_2^+,\allowbreak v_2^-,\allowbreak v_2^+)$ be non-empty octagons with normalised parameters. Then $O_1\diamond O_2 = \mathsf{Octa}(x_3^-,\allowbreak x_3^+,\allowbreak y_3^-,\allowbreak y_3^+,\allowbreak u_3^-,\allowbreak u_3^+,\allowbreak v_3^-,\allowbreak v_3^+)$, where:
\begin{align}
x_3^- &= \max(x_1^-, x_2^- - u_1^+, v_1^- - x_2^+)\label{eqComposition1}\\
x_3^+ &= \min(x_1^+, x_2^+ - u_1^-,v_1^+ - x_2^-)\label{eqComposition2}\\
y_3^- &= \max(y_2^-, u_2^-+y_1^-, v_2^- - y_1^+)\\
y_3^+ &= \min(y_2^+,u_2^+ + y_1^+,v_2^+-y_1^-)\\
u_3^- &= \max(y_2^- - x_1^+, u_2^- + u_1^-, v_2^- - v_1^+)\\
u_3^+ &= \min(y_2^+ - x_1^-,u_2^+ + u_1^+, v_2^+ - v_1^-)\\
v_3^- &= \max(x_1^- + y_2^-, u_2^-+v_1^-,v_2^- -u_1^+)\\
v_3^+ &= \min(x_1^+ + y_2^+, u_2^+ + v_1^+,v_2^+ - u_1^-)\label{eqComposition8}
\end{align}
\end{proposition}
\begin{proof}
It is straightforward to verify that $O_1\diamond O_2 \subseteq O_3$ holds. For instance, if $(x,z)\in O_1\diamond O_2$, by definition there exists some $y$ such that $(x,y)\in O_1$ and $(y,z)\in O_2$. We thus have:
\begin{itemize}
\item $x\geq x_1^-$, since $(x,y)\in O_1$;
\item $y\geq x_2^-$, since $(y,z)\in O_2$, and $y-x\leq u_1^+$, since $(x,y)\in O_1$, and thus also $x\geq x_2^- - u_1^+$;
\item $x+y\geq v_1^-$, since $(x,y)\in O_1$, and $y\leq x_2^+$, since $(y,z)\in O_2$, and thus also  $x\geq v_1^- - x_2^+$
\end{itemize}
Together, we have thus found that $x \geq \max(x_1^-,x_1^-, x_2^- - u_1^+, v_1^- - x_2^+)$, which corresponds to the bound \eqref{eqComposition1}. We can similarly show that the  bounds \eqref{eqComposition2}--\eqref{eqComposition8} are satisfied for any $(x,z)\in O_1\diamond O_2$, from which it follows that $O_1\diamond O_2 \subseteq O_3$.

Now we show that the equality holds. Let $(x_4^-,x_4^+,y_4^-,y_4^+,u_4^-,u_4^+,v_4^-,v_4^+)$ be the normalisation of the parameters $(x_3^-,x_3^+,y_3^-,y_3^+,u_3^-,u_3^+,v_3^-,v_3^+)$, computed as in Proposition \ref{propNormalisation}. If $x_4^->x_4^+$, $y_4^->y_4^+$, $u_4^->u_4^+$ or $v_4^->v_4^+$ we have that $O_3=\mathsf{Octa}(x_4^-,x_4^+,y_4^-,y_4^+,u_4^-,u_4^+,v_4^-,v_4^+)=\emptyset$. Since we have already established that $O_1\diamond O_2 \subseteq O_3$, we then also have $O_1\diamond O_2=\emptyset$ and thus the stated result holds. Let us therefore assume throughout the remainder of this proof that $x_4^- \leq x_4^+$, $y_4^- \leq y_4^+$, $u_4^- \leq u_4^+$ and $v_4^- \leq v_4^+$. Because of Lemma \ref{lemmaVerticesOctagon}, to complete the proof, it is sufficient to show that the following points belong to $O_1\diamond O_2$:
\begin{align*}
&(x_4^-,v_4^- - x_4^-)\\
&(x_4^-,u_4^+ + x_4^-)\\
&(y_4^+ - u_4^+,y_4^+)\\
&(v_4^+ - y_4^+,y_4^+)\\
&(x_4^+,v_4^+-x_4^+)\\
&(x_4^+,u_4^- + x_4^+)\\
&(y_4^- - u_4^-,y_4^-)\\
&(v_4^- - y_4^-,y_4^-)
\end{align*}
For each such a point $(x^*,z^*)$ we need to show that there exists some $y^*\in\mathbb{R}$ such that $(x^*,y^*)\in O_1$ and $(y^*,z^*)\in O_2$. This is the case iff the following conditions are satisfied:
\begin{align}
x_1^- &\leq x^* \leq x_1^+\label{eqCompositionGurobi1}\\
y_2^- &\leq z^* \leq y_2^+\label{eqCompositionGurobi2}
\end{align}
and
\begin{align}
&\max(y_1^-, u_1^- + x^*, v_1^- - x^*, x_2^-, z^*-u_2^+, v_2^- - z^*)\notag\\
&\leq \min(y_1^+, u_1^+ + x^*, v_1^+-x^*, x_2^+, z^*-u_2^-, v_2^+-z^* ) \label{eqCompositionGurobi3}
\end{align}
We used Gurobi\footnote{\url{https://www.gurobi.com}} to verify that the conditions \eqref{eqCompositionGurobi1}--\eqref{eqCompositionGurobi3} are indeed necessarily satisfied for each of the vertices $(x^*,z^*)$.
\end{proof}

\subsection{Characterising Self-Compositions}\label{secSelfComposition}
Rather than characterising transitivity directly, we first study a more general problem. For an octagon $O$, we define $O^{(k)} = O^{k-1}\diamond O$, for $k\geq 2$, and $O^{(1)}=O=\mathsf{Octa}(x^-,\allowbreak x^+,\allowbreak y^-,\allowbreak y^+,\allowbreak u^-,\allowbreak u^+,\allowbreak v^-,\allowbreak v^+)$. Assume that the parameters of $O$ are normalised. In other words, $O^{(k)}$ is the octagon we obtain if compose $O$ with itself $k-1$ times. We now analyse the nature of the resulting octagons. Among others, this will allow us to characterise under which conditions the relation captured by $O$ is \emph{transitive}. Let us use subscripts to denote the parameters of $O^{(l)}$, i.e.\ we write $O^{(l)} = \mathsf{Octa}(x_l^-,\allowbreak x_l^+,\allowbreak y_l^-,\allowbreak y_l^+,\allowbreak u_l^-,\allowbreak u_l^+,\allowbreak v_l^-,\allowbreak v_l^+)$. Note that:
\begin{align}
x^-_{l+1} &= \max(x_l^-, x^- - u_l^+, v_l^- - x^+)\label{eqSelfComposition1}\\
x^+_{l+1} &= \min(x_l^+, x^+ - u_l^-, v_l^+ - x^-)\\
y^-_{l+1} &= \max(y^-, u^- + y_l^-, v^- - y_l^+)\\
y^+_{l+1} &= \min(y^+, u^+ + y_l^+, v^+ - y_l^-)\\
u^-_{l+1} &= \max(y^- - x_l^+, u^- + u_l^-,v^- - v_l^+)\\
u^+_{l+1} &= \min(y^+ - x_l^-,u^+ + u_l^+, v^+-v_l^-)\\
v^-_{l+1} &= \max(x_l^- + y^-, u^- + v_l^-, v^- - u_l^+)\\
v^+_{l+1} &= \min(x_l^+ + y^+, u^+ + v_l^+, v^+ - u_l^-)\label{eqSelfComposition8}
\end{align}
We now have the following results.
\begin{lemma}\label{lemmaSelfCompositionEmpty}
If $u^+<0$ or $u^->0$ then there exists some $k\in\mathbb{N}$ such that $O^{(k)}=\emptyset$.
\end{lemma}
\begin{proof}
Assume that $u^+<0$; the case for $u^->0$ is entirely analogous.
From the characterisation of composition \eqref{eqSelfComposition1}--\eqref{eqSelfComposition8}, we know that for each $l\geq 1$ we have:
\begin{align*}
x^-_{l+1} &\geq x^- - u_l^+\\
x^+_{l+1} &\leq x^+_l\\
u_{l+1}^+ & \leq u^+ + u_{l}^+
\end{align*}
If $u^+<0$, it follows that $x^-_{l+1} \geq x^- + (l-1) |u^+|$. For sufficiently large $l$ we thus have $x^-_l > x^+$ meaning that $O^{(l)}=\emptyset$.
\end{proof}

\begin{lemma}\label{lemmaSelfComposition1}
Suppose that the parameters of the octagon are normalised.
If $u^- \leq 0$ and $u^+ \geq 0$ then $x_l^+ = x_2^+$, $x_l^- = x_2^-$, $y_l^+ = y_2^+$ and $y_l^- = y_2^-$ for all $l\geq 2$. Furthermore, for all $l\geq 1$ we have $u_l^-\leq 0$, $u_l^+\geq 0$, $v_l^-\leq v^-$ and $v_l^+ \geq v^+$.
\end{lemma}
\begin{proof}
Suppose $u^- \leq 0$, $u^+ \geq 0$ and $v^- < v^+$. 
We first show by induction that the following are true for $l\geq 1$:
\begin{align}
x_l^- &\leq \max(x^-,v^- - x^+) \label{eqProofOctagonSelfCompositionMonotonicA}\\
x_l^+ &\geq \min(x^+,v^+ - x^-) \\
y_l^- &\leq \max(y^-,v^- - y^+) \\
y_l^+ &\geq \min(y^+,v^+ - y^-) \\
u_l^- &\leq 0\\
u_l^+ &\geq 0\\
v_l^- &\leq v^- \label{eqProofOctagonSelfCompositionMonotonicB1}\\
v_l^+ &\geq v^+ \label{eqProofOctagonSelfCompositionMonotonicB}
\end{align}
Note that these inequalities are trivially satisfied for $l=1$. Now suppose the induction hypothesis is satisfied for $l=k$; we show that the hypothesis is also satisfied for $l=k+1$.
\begin{itemize}
\item We have $x_{k+1}^- \leq \max(x_k^-,v_k^- - x^+)$ since $u_k^+\geq 0$ by the induction hypothesis. From the induction hypothesis we furthermore have $x_k^-\leq \max(x^-,v^- - x^+)$ and $v_k^-\leq v^-$. We thus find $x_{k+1}^- \leq \max(x^-,v^- - x^+)$.
\item Entirely analogously, we find $x_{k+1}^+ \geq \min(x^+,v^+ - x^-)$
\item We now show that $u_{k+1}^-\leq 0$. First, we have $y^- - x_k^+ \leq y^- - \min(x^+,v^+ - x^-)$. Since the parameters are normalised, we have $y^- - x^+ \leq u^- \leq 0$ and $y^- - v^+ + x^- \leq v^- -v^+$ and by assumption we have $v^- - v^+ <0$. We thus have $y^- - x_k^+ \leq 0$. By the induction hypothesis we also immediately find $u^- + u_k^- \leq 0$. Finally, we also find $v^- - v_k^+ \leq v^- - v^+ \leq 0$. We thus have found $u_{k+1}^-\leq 0$.
\item In the same way we show $u_{k+1}^+\geq 0$.
\item We now show $v_{k+1}^- \leq v^-$. First, we have $x_k^- + y^- \leq \max(x^-,v^- - x^+) + y^-$.  Since the parameters are normalised, we have $x^- + y^- \leq v^-$ and $v^- -x^+ +y^- \leq v^- + u^- \leq v^-$. From the induction hypothesis we also trivially find $u^- + v_k^- \leq u^- + v^- \leq v^-$ and $v^- - u_k^+\leq v^-$. We have thus found $v_{k+1}^- \leq v^-$.
\item In the same way, we find $v_{k+1}^+ \geq v^+$.
\end{itemize}
As a consequence of \eqref{eqProofOctagonSelfCompositionMonotonicA}--\eqref{eqProofOctagonSelfCompositionMonotonicB}, we now immediately find that for all $l\geq 2$ we have
\begin{align*}
x_l^- &= \max(x^-,v^- - x^+)\\
x_l^+ &= \min(x^+,v^+ - x^-)
\end{align*}
\end{proof}

\begin{lemma}\label{lemmaSelfCompositionSquare}
Suppose $u^- < 0$, $u^+ > 0$ and $v^- < v^+$. Then there exists some $m\in\mathbb{N}$ such that $O^{(l)}= [x_2^-,x_2^+]\times [y_2^-,y_2^+]$ for all $l\geq m$. 
\end{lemma}
\begin{proof}
Since $u^+>0$ and $u^- = u_1^- <0$, we have:
\begin{align}\label{eqlemmaSelfCompositionSquare1}
(v_2^+ = x_1^+ + y^+) \vee (v_2^+ > v_1^+)
\end{align}
and in the same way, we also find
\begin{align}\label{eqlemmaSelfCompositionSquare2}
(v_2^- = x_1^- + y^-) \vee (v_2^- < v_1^-)
\end{align}
Since $u^-<0$ we have $u^- + u_{2}^- < u_{2}^-$. Furthermore,  from $v_{2}^+ > v_{1}^+$ it follows that  $v^- - v_{2}^+ < v^- - v_{1}^+ \leq u_2^-$. Together, this implies:
\begin{align*}
(v_{2}^+ > v_{1}^+) \Rightarrow (u_{3}^- = y^- - x_2^+) \vee (u_{3}^- < u_{2}^-)
\end{align*}
More generally, in the same way we can show:
\begin{align*}
&(v_{l+1}^+ >  v_l^+) \Rightarrow (u_{l+2}^- = y^- -x_2^+) \vee (u_{l+2}^- < u_{l+1}^-)\\
&(v_{l+1}^- <  v_l^-) \Rightarrow (u_{l+2}^+ = y^+ -x_2^-) \vee (u_{l+2}^+ > u_{l+1}^+)\\
&(u_{l+1}^+ >  u_l^+) \Rightarrow (v_{l+2}^- = x_2^- + y^-) \vee (v_{l+2}^- < v_{l+1}^-)\\
&(u_{l+1}^- <  u_l^-) \Rightarrow (v_{l+2}^+ = x_2^+ + y^+) \vee (v_{l+2}^+ > v_{l+1}^+)
\end{align*}
Next, we show:
\begin{align*}
(v_{l+1}^+ \geq x_2^+ + y_2^+) &\Rightarrow (u_{l+2}^- \leq  y_2^- - x_2^+) \vee (u_{l+2}^- < u_{l+1}^-) \\
(v_{l+1}^- \leq x_2^- + y_2^-) &\Rightarrow (u_{l+2}^+ \geq  y_2^+ - x_2^-) \vee (u_{l+2}^+ > u_{l+1}^+)\\
(u_{l+1}^+ \geq  y_2^+ - x_2^-) &\Rightarrow (v_{l+2}^- \leq x_2^- + y_2^-) \vee (v_{l+2}^- < v_{l+1}^-)\\
(u_{l+1}^- \leq  y_2^- - x_2^+) &\Rightarrow (v_{l+2}^+ \geq x_2^+ + y_2^+) \vee (v_{l+2}^+ > v_{l+1}^+)
\end{align*}
Assume $v_{l+1}^+ \geq x_2^+ + y_2^+$. First suppose $u_{l+2}^-=y^+-x_{l+1}^- = y^+ - x_2^-$. Noting that $y^+\geq y_2^+$ we clearly have $u_{l+2}^- \leq  y_2^- - x_2^+$. Now suppose $u_{l+2}^-=u^- + u_{l+1}^-$. Since $u^-< 0$ we clearly have $u_{l+2}^-< u_{l+1}^-$. Finally, suppose $u_{l+2}^-=v^- -v_{l+1}^+$. Then we find:
\begin{align*}
u_{l+2}^- &= v^- -v_{l+1}^+
\leq v^- - x_2^+ - y_2^+
\leq y_3^- -x_2^+
= y_2^- -x_2^+
\end{align*}
We thus have that in any case, either $(u_{l+2}^- \leq  y_2^- - x_2^+)$ or $(u_{l+2}^- < u_{l+1}^-)$ holds. The other three assertions are shown in entirely the same way. Together with \eqref{eqlemmaSelfCompositionSquare1}--\eqref{eqlemmaSelfCompositionSquare2}, this means that for even values of $l$, we have
\begin{align}
&(u_{l+1}^- \leq  y_2^- - x_2^+) \vee (u_{l+1}^- < u_{l}^-) \label{eqlemmaSelfCompositionSquareB1}\\
&(u_{l+1}^+ \geq  y_2^+ - x_2^-) \vee (u_{l+1}^+ > u_{l}^+)\label{eqlemmaSelfCompositionSquareB2}\\
&(v_{l}^- \leq x_2^- + y_2^-) \vee (v_{l}^- < v_{l-1}^-)\label{eqlemmaSelfCompositionSquareB3}\\
&(v_{l}^+ \geq x_2^+ + y_2^+) \vee (v_{l}^+ > v_{l-1}^+)\label{eqlemmaSelfCompositionSquareB4}
\end{align}
Now we consider odd values of $l$. If $u_2^+ = y^+ - x_1^-$ or $u_2^+ > u_1^+$ holds, then we find that \eqref{eqlemmaSelfCompositionSquareB2} holds, as well as \eqref{eqlemmaSelfCompositionSquareB3}  for $l\geq 3$. Now suppose $u_2^+ < y^+ - x_1^-$ and $u_2^+ \leq u_1^+$. This is only possible if $u_2^+ = v^+-v_1^-$. If $v_3^- = x_2^- + y^-$ or $v_3^- < v_2^-$ holds, then we find that \eqref{eqlemmaSelfCompositionSquareB3} \eqref{eqlemmaSelfCompositionSquareB2} hold for odd $l$ with $l\geq 3$. Now suppose we have $v_3^- > x_2^- + y^-$ and $v_3^- \geq v_2^-$. This is only possible if $v_3^- = v^- - u_2^+ = v^- - v^+ + v_1^-$. Continuing in this way, we find that there must be an odd $l$ for which \eqref{eqlemmaSelfCompositionSquareB2} and \eqref{eqlemmaSelfCompositionSquareB3} are satisfied. Indeed, suppose this were not the case. We then find:
\begin{align*}
u_4^+ &= v^+ - ( v^- - v^+ + v_1^-) = 2(v^+ - v^-) \\
v_5^- &= v^- - 2(v^+ - v^-) = 3v^- - 2v^+\\
u_6^+ &= v^+ - (3v^- - 2v^+) = 3(v^+ - v^-)\\
v_7^- &= v^- - 3(v^+ - v^-) = 4v^- - 3v^+\\
\ldots & 
\end{align*}
In particular, since $v^-<v^+$, we find for odd $l$ that $v_{l+2}^- < v_{l}^-$ and $u_{l+3}^+ > u_{l+1}^+$. This means that there must be some odd $l$ for which $v_l^- = x_2^- + y^-$. This implies that \eqref{eqlemmaSelfCompositionSquareB2} and  \eqref{eqlemmaSelfCompositionSquareB3} are satisfied for sufficiently large odd $l$. In the same way, we find that this is the case for \eqref{eqlemmaSelfCompositionSquareB1} and  \eqref{eqlemmaSelfCompositionSquareB4}. Putting everything together, we have that \eqref{eqlemmaSelfCompositionSquareB1}--\eqref{eqlemmaSelfCompositionSquareB4} are satisfied for all sufficiently large $l$. This, in turn, implies that for sufficiently large $l$, all of the bounds $u_l^-,u_l^+,v_l^-,v_l^+$ become trivial. This shows that for sufficiently large $l$, we have $O^{(l)}=[x_l^-,x_l^+]\times [y_l^-,y_l^+]$, and thus by Lemma \ref{lemmaSelfComposition1}, we also find $O^{(l)}=[x_2^-,x_2^+]\times [y_2^-,y_2^+]$.
\end{proof}

\begin{lemma}\label{lemmaSelfCompositionUL0}
Suppose $u^- = 0$, $u^+ > 0$ and $v^- < v^+$. For each $l\geq 1$ we have $u_l^-=0$, $v_l^-=v^-$ and $v_l^+=v^+$. Moreover, there exists some $m$ such that either $u_l^+=y^+-x_2^-$ for every $l\geq m$ or $u_l^+ = v^+-v^-$ for every $l\geq m$. In particular this also means that $O^{(m)}=O^{(l)}$ for all $l\geq m$. 
\end{lemma}
\begin{proof}
We first show that $u_l^-=0$ for all $l$. Recall that:
$$
u_{l+1}^- = \max(y^--x_l^+,u^-+u_l^-,v^--v_l^+)
$$
From Lemma \ref{lemmaSelfComposition1} we know that $v_l^+\geq v^+$, and since we assumed $v^-<v^+$ we thus have $v^--v_l^+ \leq 0$. We thus straightforwardly find by induction that:
$$
u_{l+1}^- = \max(y^--x_l^+,0)
$$
Moreover, from Lemma \ref{lemmaSelfComposition1} we know that $u_{l+1}^-\leq 0$. Together, we can thus conclude $u_l^-=0$ for all $l\geq 1$.

Next we show that $v_l^+=v^+$ for all $l$. First recall that:
$$
v_{l+1}^+ = \min(x_l^+ + y^+, u^+ + v_l^+, v^+-u_l^-)
$$
Since we already established that $u_l^-=0$, this simplifies to
$$
v_{l+1}^+ = \min(x_l^+ + y^+, u^+ + v_l^+, v^+)
$$
Moreover, from Lemma \ref{lemmaSelfComposition1} we know that $v_l^+\geq v^+$ and by assumption we have $u^+>0$. Thus we obtain:
$$
v_{l+1}^+ = \min(x_l^+ + y^+, v^+)
$$
Finally, since we know from \ref{lemmaSelfComposition1} that $v_{l+1}^+\geq v^+$ we find $v_{l+1}=v^+$.

Now we show $v_l^-=v^-$ for all $l$. First recall that
$$
v_{l+1}^- = \max(x_l^- + y^-, u^- + v_l^-, v^--u_l^+)
$$
which simplifies to
$$
v_{l+1}^- = \max(x_l^- + y^-, v_l^-, v^- -u_l^+)
$$
Since $u_l^+\geq 0$ by Lemma \ref{lemmaSelfComposition1}, by induction we find
$$
v_{l+1}^- = \max(x_l^- + y^-, v^-)
$$
Finally, since $v_{l+1}^-\leq v^-$ by Lemma \ref{lemmaSelfComposition1}, we obtain $v_{l+1}^-=v^-$.

Finally, we consider $u_l^+$. Recall that:
$$
u_{l+1}^+ = \min(y^+ - x_l^-,u^+ + u_l^+,v^+ - v_l^-)
$$
which simplifies to
$$
u_{l+1}^+ = \min(y^+ - x_l^-,u^+ + u_l^+,v^+ - v^-)
$$
If $u_{l+1}^+=u^+ + u_l^+$ then clearly we have $u_{l+1}^+ > u_l^+$. For sufficiently large $l$, we must thus have:
$$
u_{l}^+ = \min(y^+ - x_2^-,v^+ - v^-)
$$
In other words, either $u_l^+$ is trivial or $u_l^+ = v^+-v^-$.
\end{proof}

In entirely the same way, we find:
\begin{lemma}\label{lemmaSelfCompositionUL0sym}
Suppose $u^- < 0$, $u^+ = 0$ and $v^- < v^+$. Then for each $l\geq 1$ we have $u_l^+=0$, $v_l^-=v^-$ and $v_l^+=v^+$. Moreover, there exists some $m$ such that either $u_l^-=y^--x_2^+$ for every $l\geq m$ or $u_l^- = v^--v^+$ for every $l\geq m$. In particular this also means that $O^{(m)}=O^{(l)}$ for all $l\geq m$. 
\end{lemma}

\begin{lemma}\label{lemmaSelfCompositionUU0}
Suppose $u^- = u^+ = 0$. Then for each $l\geq 1$ we have $O^{(l)}=O^{(1)}$.
\end{lemma}
\begin{proof}
If $u^-=u^+-0$, and assuming the initial parameters are normalised, we have $O^{(1)} = \{(x,x) \,|\, x\in[x^-,x^+]\}$. From this characterisation, it immediately follows that $O^{(1)}\circ O^{(1)}=O^{(1)}$.
\end{proof}

\begin{lemma}\label{lemmaSelfCompositionVequal}
Suppose $u^-\leq 0$, $u^+\geq 0$ and $v^- = v^+$. For even $l$ we have:
$$
O^{(l)} = \{(x,x) \,|\, x\in[\max(x^-,v-x^+),\min(x^+,v-x^-)]\}
$$
while for odd $l\geq 3$ we have:
$$
O^{(l)} = \{(x,v-x) \,|\, x\in[\max(x^-,v-x^+),\min(x^+,v-x^-)]\}
$$
\end{lemma}
\begin{proof}
Let us write $v$ for $v^-=v^+$. Assuming the initial parameters are normalised, we have:
$$
O^{(1)} = \{(x,v-x) \,|\, x\in[x^-,x^+]\}
$$
and thus
$$
O^{(1)}\diamond O^{(1)} = \{(x,x) \,|\, x\in[x^-,x^+] \text{ and } v-x\in[x^-,x^+]\}
$$
In other words:
$$
O^{(2)} = \{(x,x) \,|\, x\in[\max(x^-,v-x^+),\min(x^+,v-x^-)]\}
$$
Continuing in this way, we find
$$
O^{(l)} = \{(x,v-x) \,|\, x\in[\max(x^-,v-x^+),\min(x^+,v-x^-)]\}
$$
for odd $l\geq 3$ and
$$
O^{(l)} = \{(x,x) \,|\, x\in[\max(x^-,v-x^+),\min(x^+,v-x^-)]\}
$$
for even $l\geq 2$.
\end{proof}

Let us now show the main result about octagons capturing transitive relations.

\begin{proposition}
Let $O=\mathsf{Octa}(x^-,\allowbreak x^+,\allowbreak y^-,\allowbreak y^+,\allowbreak u^-,\allowbreak u^+,\allowbreak v^-,\allowbreak v^+)$ and assume that these parameters are normalised. It holds that $O\diamond O = O$ iff one of the following conditions is satisfied:
\begin{itemize}
\item $O=\emptyset$;
\item $O = [x^-,x^+] \times [y^-,y^+]$, $x^+ \geq y^-$ and $x^- \leq y^+$;
\item $O = \{(x,x) \,|\, x\in [x^-,x^+]\}$;
\item $u^- =0$, $u^+ >0$, $v^- < v^+$, $v^--x^+\leq x^-$, $x^+\leq v^+-x^-$, $v^--y^+\leq y^-$, $y^+\leq v^+-y^-$ and  $u^+ = \min(y^+ - x^-,v^+ - v^-)$;
\item $u^- < 0$, $u^+=0$, $v^- < v^+$, $v^--x^+\leq x^-$, $x^+\leq v^+-x^-$, $v^--y^+\leq y^-$, $y^+\leq v^+-y^-$ and  $u^- = \max(y^- - x^+,v^- - v^+)$.
\end{itemize}
\end{proposition}
\begin{proof}
If $u^+<0$ or $u^->0$, we know from Lemma \ref{lemmaSelfCompositionEmpty} that $O\diamond O=O$ can only hold if $O=\emptyset$. Furthermore, $O=\emptyset$ is clearly sufficient for $O\diamond O=O$ to hold.

If $u^-<0$, $u^+>0$ and $v^-<v^+$, we know from Lemma \ref{lemmaSelfCompositionSquare} that $O\diamond O=O$ can only hold if $O=[x^-,x^+]\times [y^-,y^+]$ and moreover $x^-=x^-_2$, $x^+=x^+_2$,  $y^-=y^-_2$, $y^+=y^+_2$. These latter conditions are satisfied if $x^-\geq v^--x^+$, $x^+\leq v^+-x^-$, $y^-\geq v^--y^+$ and $y^+\leq v^+-y^-$. From $O=[x^-,x^+]\times [y^-,y^+]$, together with the assumptions that the parameters are normalised, we furthermore know that $v^-=x^- + y^-$ and $v^+=x^+ +y^+$. We thus arrive at the conditions $x^- \geq x^-+y^- -x^+$, $x^+\leq x^++y^+-x^-$,  $y^- \geq x^-+y^- -y^+$ and $y^+\leq x^++y^+-y^-$. These conditions are equivalent with $x^+\geq y^-$ and $x^-\leq y^+$. Finally, note that we can also trivially verify that the conditions stated in the second bullet point are sufficient for $O\diamond O=O$ to hold.  

If $u^-=u^+=0$ we have $O=\{(x,x)\,|\, x\in[x^-,x^+]\}$ and we clearly have $O\diamond O=O$.

If $u^-=0$, $u^+>0$ and $v^-<v^+$, then we know from Lemma \ref{lemmaSelfCompositionUL0} that $O\diamond O=O$ can only hold if either $u^+=y^+-x^-$ or $u^+=v^+-v^-$. Furthermore, we also need to ensure that $x^-=x_2^-$, $x^+=x_2^+$, $y^-=y_2^-$ and $y^+=y_2^+$, which is the case iff $v^--x^+\leq x^-$, $x^+\leq v^+-x^-$, $v^--y^+\leq y^-$ and $y^+\leq v^+-y^-$. We thus clearly have that the conditions stated in the fourth bullet point are necessary. To see why they are also sufficient, we can use \eqref{eqSelfComposition1}--\eqref{eqSelfComposition8} to verify that $O^{(2)}=O$. We trivially find $x_2^-=x^-$, $x_2^+=x^+$, $y_2^-=y^-$ and $y_2^+=y^+$. To see why $u_2^-=u^-=0$, note that $u^-+u_1^-=0$ and $v^--v^+<0$. Furthermore, $y^--x_1^+ \leq 0$ follows from the assumption that $x^- \geq v^- - x^+$. Indeed, together with the observation that $v^-\geq x^-+y^-$ we find $x^-\geq x^- + y^- -x^+$ which is equivalent with $x^+\geq y^-$. We trivially have $u_2^+=u^+$, since we assumed $u^+=\min(y^+-x^-,v^+-v^-)$ and $u^++u^+ > u^+$ since $u^+>0$. Next, we also have $v^-_2=v^-$ because $u^-+v_1^-=v^-$, $v^--u^+<v^-$, and we furthermore have $v^-\leq x^-+y^-$ since the parameters are normalised. In the same way we also find $v_2^+=v^+$.

The case where $u^-<0$, $u^+-0$ and $v^-<v^+$ is entirely analogous.

If $v^-=v^+$, it follows from Lemma \ref{lemmaSelfCompositionVequal} that $O$ can $O\diamond O=0$ can only hold in the trivial case where $O=\{(x,x)\}$ for some $x\in\mathbb{R}$, which was already covered since we then have $u^-=u^+=0$.
\end{proof}

\begin{proposition}
If $v^-<v^+$ then there exists some $m\geq 1$ such that $O^{(m)}\diamond O=O^{(m)}$.
\end{proposition}
\begin{proof}
This immediately follows from Lemmas \ref{lemmaSelfCompositionEmpty}, 
\ref{lemmaSelfCompositionSquare}, 
\ref{lemmaSelfCompositionUL0},
\ref{lemmaSelfCompositionUL0sym}
and \ref{lemmaSelfCompositionUU0}.
\end{proof}

\section{Expressivity of Octagon Embeddings}
\subsection{Capturing Knowledge Graphs}
The following result concerns hexagon of the following form:
\begin{align}\label{eqHexagons}
X^i_r = \{(x,y) \,|\, &u_i^- \leq y-x \leq u_i^+, \\ 
&x_i^- \leq x \leq x_i^+, y_i^- \leq y \leq y_i^+\} \notag
\end{align}
\begin{proposition}
Let $X_r$ be a hexagon region of the form \eqref{eqHexagons} and let $\mathbf{e},\mathbf{f}\in \mathbb{R}^n$. If $\mathbf{e}\oplus\mathbf{f}\in X_r$ and $\mathbf{f}\oplus\mathbf{e}\in X_r$ then we also have $\mathbf{e}\oplus\mathbf{e}\in X_r$.
\end{proposition}
\begin{proof}
Let $\mathbf{e}=(e_1,\ldots,e_n)$ and $\mathbf{f}=(f_1,\ldots,f_n)$ be the embeddings of entities $e$ and $f$. Since $(e,r,e)\notin\mathcal{G}$ there must exist some coordinate $i$ such that one of the following holds:
\begin{align*}
e_i &< x_i^- &
e_i &> x_i^+ &
e_i &< y_i^- &
e_i &> y_i^+ &
0 &< u_i^- &
0 &> u_i^+
\end{align*}
where $(x_i^-,x_i^+,y_i^-,y_i^+,u_i^-,u_i^+)$ are the parameters of the hexagon representing relation $r$ in coordinate $i$.
If $e_i < x_i^-$ or $e_i > x_i^+$ were to hold, then the triple $(e,r,f)$ would not be satisfied in coordinate $i$ either.
Similarly, if $e_i < y_i^-$ or $e_i> y_i^+$ were the case, then $(f,r,e)$ would not be satisfied in coordinate $i$. Suppose that $0< u_i^-$. Since $(e,r,f)$ holds, we know that $f_i-e_i\geq u_i^->0$, or in other words $f_i>e_i$. But because $(f,r,e)$ holds, we also know that $e_i-f_i\geq u_i^->0$, or in other words $e_i>f_i$, a contradiction. We similarly arrive at a contradiction if $0 > u_i^+$ were to hold.
\end{proof}

\begin{proposition}\label{propFullyExpressive}
Let $\mathcal{G}\subseteq \mathcal{E}\times\mathcal{R}\times\mathcal{E}$ be a knowledge graph. There exists an octagon embedding $\gamma$ such that $\mathcal{G}=\mathcal{G}_{\gamma}$.
\end{proposition}
\begin{proof}
We represent each entity $e$ as a vector $\gamma(e)=\mathbf{e}$ of dimension $n=|\mathcal{R}|\cdot |\mathcal{E}|$. In particular, each coordinate in $\mathbf{e}$ is associated with an entity $f\in\mathcal{E}$ and relation $r\in\mathcal{R}$; let us write $e_{r,f}$ for that coordinate. We show that there exists an octagon based embedding, which satisfies a triple $(e,r,f)$ iff it belongs to $\mathcal{G}$. In particular, we choose the coordinates of $\mathbf{e}$ as follows:
\begin{align*}
e_{r,f} = 
\begin{cases}
0 & \text{if $e=f$}\\
1 & \text{if $e\neq f$ and $(f,r,e)\in \mathcal{G}$}\\
2 & \text{otherwise}
\end{cases}
\end{align*}
Let us write $O^s_{r,e}$ for the octagon representing relation $r$ in coordinate associated with entity $e$ and relation $s$. For all coordinates where $s\neq r$, we choose $O^s_{r,e} = [0,2]\times [0,2]$. The octagons $O^r_{r,e}$ are chosen as follows:
\begin{itemize}
\item If $(e,r,e)\in \mathcal{G}$ we choose $O^r_{r,e}$ as the pentagon with vertices $(0,0), (0,1), (1,2), (2,2), (2,0)$. We will refer to this degenerate octagon as $O^+$.
\item Otherwise, we choose $O^r_{r,e}$ as the pentagon with vertices $(0,1), (1,2),(2,2), (2,0), (1,0)$. We will refer to this degenerate octagon as $O^-$.
\end{itemize}
These two octagons are depicted in Figure \ref{figFullyExpressive}. 

Suppose $(e,r,f)\in \mathcal{G}$. We need to show that $(e_{s,g},f_{s,g})\in O^r_{s,g}$ for every $s\in\mathcal{R}$ and $g\in\mathcal{E}$. If $s\neq r$ then this is trivially satisfied. Let us now consider the cases where $s=r$.
\begin{itemize}
\item If $g\neq e$, we have $e_{r,g}\in \{1,2\}$ and thus $(e_{r,g},f_{r,g})\in O^r_{r,g}$, regardless of the value of $f_{r,g}\in \{0,1,2\}$ and whether $O^r_{r,g}=O^+$ or $O^r_{r,g}=O^-$.
\item Now suppose $g=e$ with $e\neq f$. Then $e_{r,g}=e_{r,e}=0$. Since $(e,r,f)\in \mathcal{G}$, we have $f_{r,g}=1$. We thus have $(e_{r,g},f_{r,g})\in O^r_{r,g}$, regardless of whether $O^r_{r,g}=O^+$ or $O^r_{r,g}=O^-$.
\item Finally, suppose $g=e=f$. Since we assumed that $(e,r,f)=(e,r,e)\in \mathcal{G}$ we have that $O^r_{r,g}=O^+$. We also have $e_{r,g}=f_{r,g}=e_{r,e}=0$ and thus $(e_{r,g},f_{r,g})\in O^+$.
\end{itemize}
Now suppose $(e,r,f)\notin \mathcal{G}$.
We need to show that there exists some $g\in \mathcal{E}$ such that $(e_{r,g},f_{r,g})\notin O^r_{r,g}$. 
If $e=f$ we know that $O^r_{r,e}=O^-$ since we assumed $(e,r,e)\notin \mathcal{G}$ . Moreover, we have $e_{r,e}=0$ and thus $(e_{r,e},e_{r,e})\notin  O^-$. Now assume $e\neq f$. We show that $(e_{r,g},f_{r,g})\notin O^r_{r,g}$ for $g=e$. In other words, we show that $(e_{r,e},f_{r,e})\notin O^r_{r,e}$. We have $e_{r,e}=0$ and $f_{r,e}=2$. Regardless of whether  $O^r_{r,g}=O^+$ or $O^r_{r,g}=O^-$ we thus find $(e_{r,g},f_{r,g})\notin O^r_{r,g}$.
\end{proof}

\begin{figure}
\centering
\begin{subfigure}{0.49\columnwidth}
\centering
\includegraphics[width=110pt]{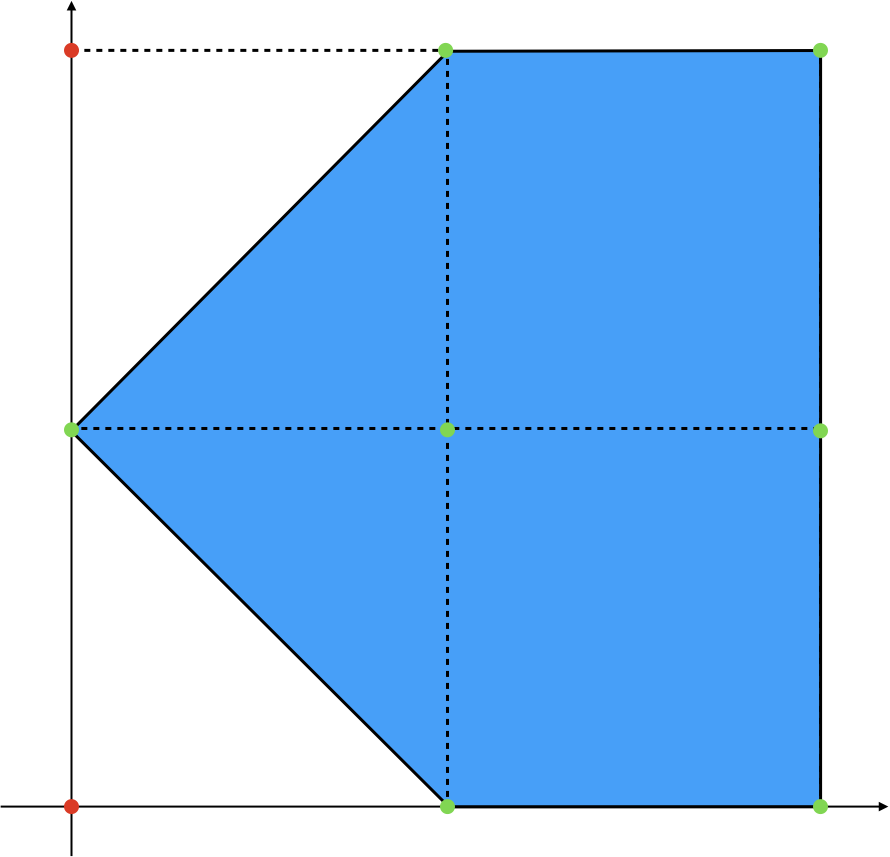}
\caption{$O^-$}
\end{subfigure}
\begin{subfigure}{0.49\columnwidth}
\centering
\includegraphics[width=110pt]{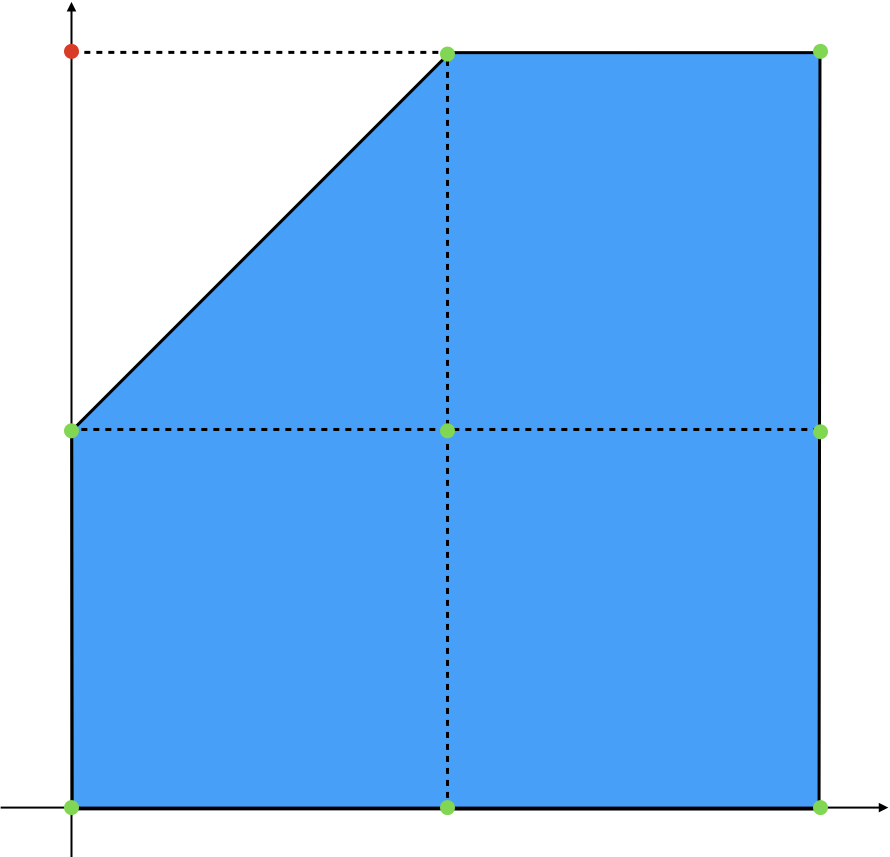}
\caption{$O^+$}
\end{subfigure}
\caption{Illustration of the octagons constructed in the proof of Proposition \ref{propFullyExpressive}.\label{figFullyExpressive}}
\end{figure}

\subsection{Capturing Rules: Negative Results}
The following result concerns region based embeddings of the form
\begin{align}\label{eqCoordinateWise}
X_r = \{(x_1{,}...,x_{n},y_1{,}...,y_n) \,|\, \forall i \,{\in}\,\{1{,}...,n\}.(x_i,y_{i})\in X^i_r\}
\end{align}
and
\begin{align}\label{eqBoxE}
X_r = \{&(x_1{,}...,x_{2n},y_1{,}...,y_{2n}) \,|\,\\
&\forall i\in\{1{,}...,n\}\,.\,(x_i,y_{n+i})\,{\in}\, A^i_r \wedge (y_i,x_{n+i})\,{\in}\, B^i_r\}\notag
\end{align}

\begin{proposition}\label{propNoMutualExclusion}
There exists a consistent set of hierarchy, intersection and mutual exclusion rules $\mathcal{K}$ such that every convex embedding of the form \eqref{eqCoordinateWise} or \eqref{eqBoxE} which satisfies $\mathcal{K}$ also satisfies some hierarchy rule which is not entailed by $\mathcal{K}$. 
\end{proposition}
\begin{proof}
Let $\mathcal{K}$ contain the following rules: 
\begin{align*}
&\{s_i(X,Y) \rightarrow r_j(X,Y) \,|\, i\,{\in}\, \{1,2,3,4\}, j\,{\in}\, \{1,2,3,4\}{\setminus} \{i\}\}
\end{align*}
as well as the following rules:
\begin{align*}
r_2(X,Y) \wedge r_3(X,Y) \wedge r_4(X,Y) &\rightarrow s_1(X,Y)\\
s_1(X,Y) \wedge r_1(X,Y) & \rightarrow \bot
\end{align*}
Note that $\mathcal{K}$ is indeed consistent. In fact, we can capture the rules in $\mathcal{K}$ using a region-based embedding $\gamma$ in $\mathbb{R}^4$ as follows:
\begin{align*}
\gamma(s_1) &= \{(0,0,0,0)\}\\
\gamma(s_2) &= \{(1,0,0,0)\}\\
\gamma(s_3) &= \{(0,1,0,0)\}\\
\gamma(s_4) &= \{(0,0,1,0)\}\\
\gamma(r_1) &= \text{CH}(\{(1,0,0,0),(0,1,0,0),(0,0,1,0)\})\\
\gamma(r_2) &= \text{CH}(\{(0,0,0,0),(0,1,0,0),(0,0,1,0)\})\\
\gamma(r_3) &= \text{CH}(\{(0,0,0,0),(1,0,0,0),(0,0,1,0)\})\\
\gamma(r_4) &= \text{CH}(\{(0,0,0,0),(1,0,0,0),(0,1,0,0)\})
\end{align*}
where we write CH to denote the convex hull.

Now let us consider convex embeddings $\gamma$ of the form \eqref{eqCoordinateWise}. Suppose such an embedding satisfies the rules in $\mathcal{K}$. Suppose furthermore that $\gamma(r)\neq \emptyset$ for every relation $r\in\{s_1,s_2,s_3,s_4,r_1,r_2,r_3,r_4\}$. Since $\gamma$ satisfies $s_1(X,Y) \wedge r_1(X,Y) \rightarrow \bot$, there must be some $\ell \in \{1,\ldots,n\}$ such that $X_{s_1}^{\ell}\cap X_{r_1}^{\ell}=\emptyset$, where we write $X^i_r$ for the two-dimensional convex region modelling relation $r$ in coordinate $i$. 
Since $s_1(X,Y) \rightarrow r_2(X,Y)$, $s_1(X,Y) \rightarrow r_3(X,Y)$ and $s_1(X,Y) \rightarrow r_4(X,Y)$ are satisfied, with $\gamma(s_1)\neq \emptyset$, we know that there exists some point $p_1\in X_{r_2}^{\ell}\cap X_{r_3}^{\ell} \cap X_{r_4}^{\ell}$. In the same way, we find that there exist points $p_2\in X_{r_1}^{\ell}\cap X_{r_3}^{\ell} \cap X_{r_4}^{\ell}$, $p_3\in X_{r_1}^{\ell}\cap X_{r_2}^{\ell} \cap X_{r_4}^{\ell}$ and $p_4\in X_{r_1}^{\ell}\cap X_{r_2}^{\ell} \cap X_{r_3}^{\ell}$. From Helly's theorem, it then follows that there exists a point $q\in X_{r_1}^{\ell}\cap X_{r_2}^{\ell} \cap X_{r_3}^{\ell} \cap X_{r_4}^{\ell}$. Because $\gamma$ satisfies the rule $r_2(X,Y) \wedge r_3(X,Y) \wedge r_4(X,Y) \rightarrow s_1(X,Y)$, we must then also have that $q\in X_{s_1}^{\ell}$. This is a contradiction since we assumed $X_{s_1}^{\ell}\cap X_{r_1}^{\ell}=\emptyset$ and we have $q\in X_{r_1}^{\ell}$. It follows that some of the regions $\gamma(s_1),\gamma(s_2),\gamma(s_3),\gamma(s_4)$ must be empty. But then $\gamma$ always satisfies rules which are not entailed by $\mathcal{K}$. For instance, if $\gamma(s_1)=\emptyset$ then, among others we have that $\gamma\models s_1\rightarrow r_1$.

For convex embeddings of the form \eqref{eqBoxE}, we find from the fact that $\gamma$ satisfies $s_1(X,Y) \wedge r_1(X,Y) \rightarrow \bot$ that there must be some $\ell \in \{1,\ldots,n\}$ such that $A_{s_1}^{\ell}\cap A_{r_1}^{\ell}=\emptyset$ or $B_{s_1}^{\ell}\cap B_{r_1}^{\ell}=\emptyset$. The remainder of the proof is then entirely analogous, focusing on the regions $A_r^{\ell}$ or $B_r^{\ell}$ rather than $X_r^{\ell}$.
\end{proof}

Let the octagons $O^{(l)}$ be defined as in Section \ref{secSelfComposition}.
\begin{lemma}\label{lemmaMonotoneOctagonCompositions}
Suppose $u^- \leq 0$, $u^+ \geq 0$ and $v^- \leq v^+$. For every $l\geq 2$ it holds that $O^{(l)}\subseteq O^{(l+2)}$. 
\end{lemma}
\begin{proof}
From Lemma \ref{lemmaSelfComposition1} we already know that
\begin{align*}
x_{l+2}^- &= x_l^- &
x_{l+2}^+ &= x_l^+ &
y_{l+2}^- &= y_l^- &
y_{l+2}^+ &= y_l^+
\end{align*}
For the remaining bounds, entirely similar as in the proof of Lemma \ref{lemmaSelfCompositionSquare}, we find:
\begin{align*}
&(v_{l+1}^+  \geq  v_l^+) \Rightarrow (u_{l+2}^- = y^- -x_2^+) \vee (u_{l+2}^- \leq u_{l+1}^-)\\
&(v_{l+1}^- \leq  v_l^-) \Rightarrow (u_{l+2}^+ = y^+ -x_2^-) \vee (u_{l+2}^+ \geq u_{l+1}^+)\\
&(u_{l+1}^+ \geq  u_l^+) \Rightarrow (v_{l+2}^- = x_2^- + y^-) \vee (v_{l+2}^- \leq v_{l+1}^-)\\
&(u_{l+1}^- \leq  u_l^-) \Rightarrow (v_{l+2}^+ = x_2^+ + y^+) \vee (v_{l+2}^+ \geq v_{l+1}^+)
\end{align*}
Note that $u_{l+2}^- = y^- -x_2^+$ implies $u_{l+2}^- \leq u_{l+1}^-$, and similar for the other cases. These implications can thus be simplified to:
\begin{align}
&(v_{l+1}^+  \geq  v_l^+) \Rightarrow  (u_{l+2}^- \leq u_{l+1}^-)\label{eqProofSelfComposition2IncreaseImplA1}\\
&(v_{l+1}^- \leq  v_l^-) \Rightarrow (u_{l+2}^+ \geq u_{l+1}^+)\\
&(u_{l+1}^+ \geq  u_l^+) \Rightarrow (v_{l+2}^- \leq v_{l+1}^-)\\
&(u_{l+1}^- \leq  u_l^-) \Rightarrow  (v_{l+2}^+ \geq v_{l+1}^+)
\end{align}
From Lemma \ref{lemmaSelfComposition1} we also know that $v_2^+ \geq v^+$ and $v_2^- \leq v^-$.
Putting everything together, we find that for even values of $l$, the following hold:
\begin{align}
&(u_{l+1}^- \leq u_{l}^-) \label{eqPropSelfCompositionSquareB1weak}\\
&(u_{l+1}^+ \geq u_{l}^+)\label{eqPropSelfCompositionSquareB2weak}\\
&(v_{l}^- \leq v_{l-1}^-)\label{eqPropSelfCompositionSquareB3weak}\\
&(v_{l}^+ \geq v_{l-1}^+)\label{eqPropSelfCompositionSquareB4weak}
\end{align}
Let $l\geq 2$. We show that $u_{l+2}^- \leq u_l^-$. 
\begin{itemize}
\item First assume that $l$ is even.  Then we already know that $u_{l+1}^-\leq u_l^-$. 
We can only have $u_{l+2}^- > u_{l+1}^-$ if $u_{l+2}^- = v^- - v_{l+1}^+$. From $u_{l+2}^- > u_{l+1}^-$, we also find that $v^- - v_{l+1}^+ > v^- - v_{l}^+$ has to hold, or in other words $v_l^+ > v_{l+1}^+$, which is only possible if $v_{l+1}^- = v^+ - u_l^-$. But then we find $u_{l+2}^- = v^- - v_{l+1}^+ = v^- - v^+ + u_l^- \leq u_l^-$. In any case, we thus find that $u_{l+2}^-\leq u_l^-$.
\item Now assume that $l$ is odd. From \eqref{eqPropSelfCompositionSquareB1weak} we know $u_{l+2}^-\leq u_{l+1}^-$. Suppose $u_{l+1}^- > u_l^-$. Note that $u_{l+2}^- \leq u_l^-$ is trivially satisfied if $u_{l+2}^- = y^- - x_{l+1}^+ = y^- - x_{2}^+ = y^- - x_{l-1}^+$. If $u_{l+2}^-=v^- - v_{l+1}^+$ we find $u_{l+2}^- = v^- - v_{l+1}^+ = v^- - v^+ + u_l^- \leq u_l^-$ as in the previous case. Now assume $u_{l+2}^-=u^-+u_l^-$. From $u_{l+1}^-> u_{l}^-$, using \eqref{eqProofSelfComposition2IncreaseImplA1}, we find $v_{l}^+ < v_{l-1}^+$. This is only possible if $u_{l+1}^- = v^--v_l^+$ and $v_l^+ = v^+ - u_{l-1}^-$. We find
$u_{l+2}^- = u^- + u_{l+1}^- = u^- + v^- - v^+ + u_{l-1}^-$. If $u_{l+2}^- > u_l^-$ were to hold, we would thus have $u^- + v^- - v^+ + u_{l-1}^- > u_l^- \geq u^- + u_{l-1}^-$, which implies $v^- > v^+$, a contradiction.
\end{itemize}
In entirely the same way, we can show that $u_{l+2}^+ \geq u_l^+$, $v_{l+2}^- \leq v_l^-$ and $v_{l+2}^+ \geq v_l^+$.
\end{proof}

\begin{proposition}\label{propNoArbitraryComposition}
Let $\mathcal{K} = \{r_1(X,Y) \wedge r_1(Y,Z) \rightarrow r_2(X,Z), r_2(X,Y) \wedge r_2(Y,Z) \rightarrow r_3(X,Z)\}$. It holds that every octagon embedding $\gamma$ which satisfies $\mathcal{K}$ also satisfies some composition rule which is not entailed by $\mathcal{K}$.
\end{proposition}
\begin{proof}
Let $\gamma$ be a coordinate-wise octagon embedding satisfying the rules in $\mathcal{K}$ and let $\gamma(r_1) = X_{r_1} = [O_1,\ldots,O_n]$. First assume that there is some $i$ such that $O_i = \mathsf{Octa}(x^-,\allowbreak x^+,\allowbreak y^-,\allowbreak y^+,\allowbreak u^-,\allowbreak u^+,\allowbreak v^-,\allowbreak v^+)$ with $u^+<0$ or $u^->0$. Then we know from Lemma \ref{lemmaSelfCompositionEmpty} that there exists some $k\in\mathbb{N}$ such that $O_i^{(k)}=\emptyset$. This also means that $X_{r_1}^{(k)}=\emptyset$, where we define $X_{r_1}^{(l)}=X_{r_1}^{(l-1)}\diamond X_{r_1}$ for $l\geq 2$ and $X_{r_1}^{(1)}=X_{r_1}$. Among others, we find that $\gamma$ satisfies the following composition rule, which is not entailed by $\mathcal{K}$:
\begin{align*}
r_1(X_1,X_2) \wedge \ldots \wedge r_1(X_k,X_{k+1}) \rightarrow r_2(X_1,X_{k+1})
\end{align*}
Now assume that none of the octagons $O_1,\ldots,O_n$ has a strictly negative $u^+$ bound or a strictly positive $u^-$ bound. From Lemma \ref{lemmaMonotoneOctagonCompositions} it then follows that:
$$
X_{r_1}\diamond X_{r_1} \subseteq X_{r_1}\diamond X_{r_1}\diamond X_{r_1}\diamond X_{r_1}
$$
It follows that $\gamma$ satisfies the following composition rule, which is not entailed by $\mathcal{K}$:
$$
r_1(X,Y) \wedge r_1(Y,Z) \rightarrow r_3(X,Z)
$$
\end{proof}

\subsection{Capturing Non-Composition Rules}
Let $\mathcal{K}$ be a rule base consisting of symmetry, inversion, hierarchy and intersection rules. We construct an octagon embedding $\gamma_{\mathcal{K}}$ satisfying these rules, without satisfying any rules not entailed by $\mathcal{K}$. Let $\mathcal{R}$ be the set of relations appearing in $\mathcal{K}$. We consider assignments $\alpha:\mathcal{R} \rightarrow \{-2,-1,0,1,2\}$ and write $\mathcal{A}$ for the set of all such assignments. We will construct embeddings with one coordinate for each assignment, writing $\alpha_i$ for the assignment associated with coordinate $i$. The following octagons will play a central role:
\begin{align*}
X_j &= \text{CH}\{(0,0),(j,j),(j,0)\}\\
Y_j &= \text{CH}\{(0,0),(j,j),(0,j)\}\\
Z &= \text{CH}\{(0,0),(0,1),(1,2),(2,2),(2,1),(1,0)\}
\end{align*}
Note that $X_j^{\textit{inv}}=Y_j$. Let $O_{r,i}$ be the octagon representing relation $r$ in coordinate $i$. We initialise these octagons as follows:
\begin{align*}
O^{(0)}_{r,i} = 
\begin{cases}
X_{\alpha_i(r)} & \text{if $\alpha_i(r)>0$}\\
Y_{|\alpha_i(r)|} & \text{if $\alpha_i(r)<0$}\\
Z & \text{otherwise}
\end{cases}
\end{align*}
We then consider the following update rules ($j\geq 1$):
\begin{align*}
O^{(j)}_{r,i} = \textit{cl}\{&O^{(j-1)}_{r,i}\\
&\cup \{O^{(j)}_{s,i}\,|\, \mathcal{K}\models s\subseteq r\}\\
&\cup \{(O^{(j)}_{s,i})^{\textit{inv}}\,|\, \mathcal{K}\models s\subseteq r^{-1}\}\\
&\cup \{O^{(j)}_{s,i} \cap O^{(j)}_{t,i} \,|\, \mathcal{K}\models s\cap t \subseteq r\}\}
\end{align*}
Clearly, in each iteration, we have $O^{(j)}_{r,i}\supseteq O^{(j-1)}_{r,i}$. Furthermore, there are only finitely many values $O^{(j)}_{r,i}$ can take (i.e.\ the closures of subsets of $\{X_1,X_2,Z,Y_1,Y_2\}$). This iterative process will thus reach a fix point after a finite number of steps. Let $O_{r,i}$ be the resulting octagons, and $\gamma_{\mathcal{K}}$ the associated octagon embedding. The following result immediately follows from the proposed construction:

\begin{lemma}
It holds that $\gamma_{\mathcal{K}}$ captures every rule in $\mathcal{K}$.
\end{lemma}

We now show that $\gamma_{\mathcal{K}}$ does not satisfy any unwanted rules.

\begin{lemma}\label{lemmaNonCompositionRulesNoUnwantedHierarchy}
Assume $\mathcal{K}\not\models r\subseteq s$. There exists a coordinate $i$ such that $O_{r,i}\not\subseteq O_{s,i}$.
\end{lemma}
\begin{proof}
Consider the following assignment:
\begin{align*}
\alpha_i(t) = 
\begin{cases}
2 & \text{if $t=r$}\\
1 & \text{otherwise}
\end{cases}
\end{align*}
We show by induction that for any relation $t$ it holds that  $X_2\subseteq O^{(j)}_{t,i}$ implies $\mathcal{K}\models r\subseteq t$ and $Y_2\subseteq O^{(j)}_{t,i}$ implies $\mathcal{K}\models r^{-1}\subseteq t$, from which it will follow in particular that $O^{(j)}_{r,i}\not\subseteq O^{(j)}_{s,i}$. The base case is trivially satisfied since we only have $X_2\subseteq O^{(j)}_{t,i}$ for $t=r$, while $Y_2\subseteq O^{(j)}_{t,i}$ does not hold for any relation.  Now suppose the result is known to hold for all octagons $O^{(j-1)}_{t,i}$. 

We now consider the inductive case. 
Suppose $X_2\subseteq O^{(j)}_{t,i}$. 
Note that all octagons are closures of subsets of $X_1,X_2,Y_1,Y_2$. For such octagons, we have $X_2\subseteq \textit{cl}(O_1,\ldots,O_k)$ iff $X_2\subseteq O_j$ for some $j$. We thus have that one of the following cases must be true.
\begin{itemize}
\item If $X_2\subseteq O^{(j-1)}_{t,i}$ then by induction we have $\mathcal{K}\models r\subseteq t$. 
\item If $X_2\subseteq O^{(j-1)}_{p,i}$ with $\mathcal{K}\models p\subseteq t$ then by induction we have $\mathcal{K}\models r\subseteq p$, and thus we also find $\mathcal{K}\models r\subseteq t$.  
\item If $Y_2\subseteq O^{(j-1)}_{p,i}$ with $\mathcal{K}\models p\subseteq t^{-1}$ then by induction we have $\mathcal{K}\models r^{-1}\subseteq p$, and thus we also find $\mathcal{K}\models r\subseteq t$. 
\item If $X_2\subseteq O^{(j-1)}_{p,i}\cap O^{(j-1)}_{q,i}$ with $p\cap q\subseteq t$ then by induction we have $\mathcal{K}\models r\subseteq p$ and $\mathcal{K}\models r\subseteq q$ and thus also 
$\mathcal{K}\models r\subseteq t$.
\end{itemize}
The proof for $Y_2\subseteq O^{(j)}_{t,i}$ is entirely analogous.
\end{proof}

\begin{lemma}
Assume $\mathcal{K}\not\models r\subseteq s^{-1}$. There exists a coordinate $i$ such that $O_{r,i}\not\subseteq O_{s,i}^{\textit{inv}}$.
\end{lemma}
\begin{proof}
We consider the same assignment as in the proof of Lemma \ref{lemmaNonCompositionRulesNoUnwantedHierarchy}:
\begin{align*}
\alpha_i(t) = 
\begin{cases}
2 & \text{if $t=r$}\\
1 & \text{otherwise}
\end{cases}
\end{align*}
We then again have that for any relation $t$ it holds that  $X_2\subseteq O^{(j)}_{t,i}$ implies $\mathcal{K}\models r\subseteq t$ and $Y_2\subseteq O^{(j)}_{t,i}$ implies $\mathcal{K}\models r^{-1}\subseteq t$. We have in particular that $X_2\subseteq O^{(j)}_{s^{-1},i}$ iff  $Y_2 \subseteq O^{(j)}_{s,i}$. The latter implies $\mathcal{K}\models r^{-1}\subseteq s$, which is equivalent to $\mathcal{K}\models r\subseteq s^{-1}$. Since we assumed $\mathcal{K}\models r\not\subseteq s^{-1}$, it follows that $X_2\not\subseteq O^{(j)}_{s^{-1},i}$ and thus $O_{r,i}\not\subseteq O_{s,i}^{\textit{inv}}$.
\end{proof}

Note that the above lemma also holds if $r=s$, so we have also established that $\gamma_{\mathcal{K}}$ does not capture any unwanted symmetry rules.

\begin{lemma}
Assume $\mathcal{K}\not\models r_1\cap r_2 \subseteq s$. There exists a coordinate $i$ such that $O_{r_1,i} \cap O_{r_2,i} \not\subseteq O_{s,i}$.
\end{lemma}
\begin{proof}
Consider the following assignment:
\begin{align*}
\alpha_i(t) = 
\begin{cases}
2 & \text{if $t\in \{r_1,r_2\}$}\\
1 & \text{otherwise}
\end{cases}
\end{align*}
The remainder of the proof is entirely analogous to the proof of Lemma \ref{lemmaNonCompositionRulesNoUnwantedHierarchy}. In particular, in the same way, we find
 by induction that for any relation $t$ it holds that  $X_2\subseteq O^{(j)}_{t,i}$ implies $\mathcal{K}\models r_1 \cap r_2 \subseteq t$ and $Y_2\subseteq O^{(j)}_{t,i}$ implies $\mathcal{K}\models r_1^{-1} \cap r_2^{-1}\subseteq t$.
\end{proof}

\begin{lemma}
For any $r_1,r_2,s\in \mathcal{R}$, there exists a coordinate $i$ such that $O_{r_1,i} \diamond O_{r_2,i} \not\subseteq O_{s,i}$.
\end{lemma}
\begin{proof}
Consider the following assignment.
\begin{align*}
\alpha_i(t) = 
\begin{cases}
0 & \text{if $t\in \{r_1,r_2\}$}\\
1 & \text{otherwise}
\end{cases}
\end{align*}
Note that we clearly have $(0,2)\notin O_{r,i}^{(l)}$ for any $r\in\mathcal{R}$ and $l\geq 0$.
However, we also have
\begin{align*}
O_{r_1,i} \diamond O_{r_2,i}\supseteq Z \diamond Z = \text{CH}\{(0,0),(0,2),(2,0),(2,2)\}
\end{align*}
Since $(0,2)\in O_{r_1,i} \diamond O_{r_2,i}$ we obtain $O_{r_1,i} \diamond O_{r_2,i} \not\subseteq O_{s,i}$.
\end{proof}

Finally, the following two lemmas, showing that $\gamma_{\mathcal{K}}$ does not capture any unwanted asymmetry and mutual exclusion rules, follow trivially from the construction.
\begin{lemma}
Let $r_1,r_2\in \mathcal{R}$. For every coordinate $i$, it holds that $O_{r_1,i} \cap O_{r_2,i} \neq \emptyset$.
\end{lemma}

\begin{lemma}
Let $r\in \mathcal{R}$. For every coordinate $i$, it holds that $O_{r,i} \cap O_{r,i}^{\textit{inv}} \neq \emptyset$.
\end{lemma}

Putting everything together, we have shown the following result.

\begin{proposition}
Let $\mathcal{K}$ be a set of symmetry, inversion, hierarchy and intersection rules. There exists a coordinate-wise octagon embedding $\gamma$ which satisfies $\mathcal{K}$, and which only satisfies those symmetry, inversion, hierarchy and intersection rules which are entailed by $\mathcal{K}$, and which does not satisfy any asymmetry, mutual exclusion and composition rules.
\end{proposition}

\subsection{Capturing Rules: Compositions}
We will consider octagons of the following form ($i+j\leq m$):
\begin{align*}
A_{i,j,m} &= \text{CH}\{(i,0),(i,i+j),(m-j,m),(m,m),(m,0)\}\\
&= \mathsf{Octa}(i,m,0,m,-m,j,i,2m)
\end{align*}

For the ease of presentation, in the following we will sometimes write rules in terms of the relational operators and set inclusions. For instance, a rule such as $r_1(X,Y)\wedge r_2(Y,Z)\rightarrow r(X,Z)$ will be written as $r_1\circ r_2 \subseteq r$.

\begin{lemma}\label{lemmaCompositionEmbeddingAcyclic}
Let $\mathcal{K}$ be a set of extended composition rules. Assume that any extended composition rule entailed by $\mathcal{K}$ is either a trivial rule of the form $r\subseteq r$ or a regular rule.
Let $\mathcal{R}$ be the set of relations appearing in $\mathcal{K}$. There exists an assignment $\tau: \mathcal{R}\rightarrow \mathbb{N}$ such that whenever $r$ appears in the body of a rule from $\mathcal{K}$ which has $s$ in the head ($r,s\in\mathcal{R}$) it holds that $\tau(r)<\tau(s)$.
\end{lemma}
\begin{proof}
Let us consider a dependency graph $G=(\mathcal{R},E)$ where the nodes of the graph correspond to the set of relations $\mathcal{R}$ and there is an edge $(r,s)\in E$ iff $\mathcal{K}$ contains a rule with $r$ in the body and $s$ in the head. If this graph is free from cycles, then we can straightforwardly construct a suitable assignment $\tau$ using topological sort. 

To conclude the proof, we show by contradiction that $G$ is always free from cycles. Suppose $G$ did contain a cycle $(r_1,r_2),\ldots,(r_k,r_1)$. Because $(r_k,r_1)\in E$, we know that there is some rule of the form
\begin{align*}
&t_1 \circ \ldots \circ r_k \circ \ldots \circ t_l \subseteq r_1
\end{align*}
Because $(r_{k-1},r_k)\in E$ we know that there is a rule which  has $r_k$ in the head and $r_{k-1}$ in the body.  We thus also have that $\mathcal{K}$ must entail a rule which has $r_1$ in the head and $r_{k-1}$ in the body. Continuing this argument, we find that $\mathcal{K}$ entails a rule with $r_1$ oin the body and $r_1$ in the head. This is a contradiction because we had assumed that any extended composition rule entailed by $\mathcal{K}$ is regular. It follows that $G$ cannot have any cycles.
\end{proof}

Let $\mathcal{K}$ be a set of regular composition rules, and assume that any extended composition rule entailed by $\mathcal{K}$ is either a trivial rule of the form $r\subseteq r$ or a regular rule. We construct an embedding capturing the rules in $\mathcal{K}$ as follows. Let us consider assignments $\alpha$ from  $\mathcal{R}$ to pairs $(i,k)$ with $1\leq i\leq k$ and $k\leq |\mathcal{R}|+1$. Let $\mathcal{A}$ be the set of all such assignments. We consider embeddings with one coordinate for each of these assignments. Let us write $\alpha_i$ for the assignment associated with coordinate $i$. We furthermore write $O_{r,i}$ for the $i\textsuperscript{th}$ coordinate of the octagon representing relation $r$. We define these octagons as follows. First, for $r\in \mathcal{R}$, with $\alpha_i(r)=(j,k)$, we define:
\begin{align*}
X_{r,i} = 
\begin{cases}
A_{j,1,k} & \text{if $j<k$}\\
\{(k,k)\} & \text{otherwise}
\end{cases}
\end{align*}
Let us write $\mathsf{DC}(\mathcal{K})$ for the deductive closure of $\mathcal{K}$. More precisely, $\mathsf{DC}(\mathcal{K})$ is the set of all extended composition rules which can be entailed from $\mathcal{K}$ and which are not trivial rules of the form $r\subseteq r$.
We can then consider the following recursive definition, which we know to be well-defined thanks to Lemma \ref{lemmaCompositionEmbeddingAcyclic} ($r\in \mathcal{R}$):
\begin{align*}
O_{r,i} = \textit{cl}\{&X_{r,i} \cup\\
&\{O_{s_1,i}\,{\diamond}\, ... \,{\diamond}\, O_{s_k,i} \,|\, (s_1 \,{\circ}\, ... \,{\circ}\, s_k \subseteq r) \in \mathsf{DC}(\mathcal{K})\}\}
\end{align*}
where the closure $\textit{cl}$ of a set of octagons $O_1,\ldots,O_k$ with $O_i=\mathsf{Octa}(x_i^-,\allowbreak x_i^+,\allowbreak y_i^-,\allowbreak y_i^+,\allowbreak u_i^-,\allowbreak u_i^+,\allowbreak v_i^-,\allowbreak v_i^+)$ is given by the octagon $O^* = \mathsf{Octa}(x_*^-,\allowbreak x_*^+,\allowbreak y_*^-,\allowbreak y_*^+,\allowbreak u_*^-,\allowbreak u_*^+,\allowbreak v_*^-,\allowbreak v_*^+)$ with $x_*^- = \min_i x_i^-$, $x_*^+ = \max_i x_i^+$, $y_*^- = \min_i y_i^-$, $y_*^+ = \max_i y_i^+$, $u_*^- = \min_i u_i^-$, $u_*^+ = \max_i u_i^+$, $v_*^- = \min_i v_i^-$ and $v_*^+ = \max_i v_i^+$. In other words, $\textit{cl}(O_1,\ldots,O_k)$ is the smallest octagon which contains $O_1,\ldots,O_k$.

Let us write $\gamma_{\mathcal{K}}$ for the octagon embedding defined above. 
The following result immediately follows from the construction of $\gamma_{\mathcal{K}}$.
\begin{lemma}\label{lemmaConstructionCompositionRulesIsSound}
It holds that $\gamma_{\mathcal{K}}$ captures every rule in $\mathcal{K}$.
\end{lemma}

We still need to show that $\gamma_{\mathcal{K}}$ only captures the extended composition rules which are entailed by $\mathcal{K}$.
We will first show this for rules of the form $r_1\circ \ldots \circ r_k\subseteq \mathcal{R}$ where all of the relations $r_1,\ldots,r_k,r$ are distinct. 
Let us consider such a composition $r_1\circ \ldots \circ r_k$, where $k\leq |\mathcal{R}|-1$ and $r_i\neq r_j$ for $i\neq j$. We associate with such a composition the following assignment $\alpha_i$.
\begin{align*}
\alpha_i(r) = 
\begin{cases}
(j,k+1) & \text{if $r=r_j$}\\
(k+1,k+1) & \text{otherwise}
\end{cases}
\end{align*}
For the ease of presentation, we will write $A_{i,j,k+1}$ as $A_{i,j}$ in the following. Furthermore, we will write $b_{u}^+(O)$ for the normalised $y^+$ bound of an octagon $O$, i.e.\  $b_{u}^+(O)=\max\{y-x\,|\, (x,y)\in O\}$. Similarly, we write $b_{x}^-(O)$ for $\min\{x \,|\, (x,y)\in O\}$ and $b_{y}^+(O)$ for $\max\{y \,|\, (x,y)\in O\}$.

With each $s\in \mathcal{R}$, we associate a subset $\mathcal{R}_s$ of $\{r_1,\ldots,r_k\}$ as follows: $r_j\in \mathcal{R}_s$ if $\mathcal{K}$ entails a rule with $r_j$ in the body and $s$ in the head. Let us write $|\mathcal{R}_s|=n_s$ and $\min\{j \,|\, r_j \in\mathcal{R}_s\}=j_s$, where we assume $j_s=k+1$ if $|\mathcal{R}_s|=\emptyset$. 

\begin{lemma}\label{lemmaCompositionBoundsUandX}
It holds that $b_u^+(O_{s,i})\leq n_s$ and $b_x^-(O_{s,i})\geq j_s$. 
\end{lemma}
\begin{proof}
We show the result by structural induction, taking advantage of the acyclic nature of the rule base (i.e.\ Lemma \ref{lemmaCompositionEmbeddingAcyclic}). Let us first consider the base case where there are no rules of the form $t_1\circ \ldots\circ t_p \subseteq s$ entailed by $\mathcal{K}$ (except for the trivial rule $s\subseteq s$). If $s=r_j$, we have  $\mathcal{R}_j=\{r_j\}$, due to the assumption about there being no rules of the form $t_1\circ \ldots\circ t_p \subseteq s$ entailed by $\mathcal{K}$. Accordingly, we have  $b_u^+(O_{s,i})=b_u^+(A_{j,1})=1=n_s$ and $b_x^-(O_{s,i})=j=j_s$. If $s\notin \{r_1,\ldots,r_k\}$, we have $n_s=0$ and $X_{s,i}=\{(k+1,k+1)\}$. It follows that $b_u^+(O_{s,i}) = 0\leq n_s$ and $b_x^-(O_{s,i}) = k+1 =j_s$.

For the inductive case, assume that for every rule $t_1\circ \ldots\circ t_p \subseteq s$ entailed by $\mathcal{K}$, apart from the trivial rule $s\subseteq s$, the result is already known to hold for $t_1,\ldots,t_p$. 
If $n_s=0$, we must have $n_{t_1}=\ldots=n_{t_p}=0$. By induction we thus have $b_u^+(O_{t_1,i})=\ldots=b_u^+(O_{t_p,i})=0$. From the characterisation of composition, it then follows that  $b_u^+(O_{t_1,i}\diamond \ldots \diamond O_{t_p,i})=0$. Furthermore, we have $X_{s,i}=\{(k+1,k+1)\}$, meaning $b_u^+(X_{s,i})=0$. We can thus conclude $b_u^+(O_s,i)=0$. Note that when $n_s=0$, the condition $b_x^-(O_{s,i})\geq j_s$ is trivial as we then have $j_s=k+1$.

Now assume $n_s>0$. Note that $j_{t_1}\geq j_s$. By induction we have $b_x^-(O_{t_1,i})\geq j_{t_1}\geq j_s$. From the characterisation of composition we also have $b_x^-(O_{t_1,i}\diamond \ldots \diamond O_{t_p,i})\geq b_x^-(O_{t_1,i}) \geq j_s$. We clearly also have $b_x^-(X_{s,i})\geq j_s$. We thus have $b_x^-(O_{s,i})\geq j_s$. Since $\mathcal{K}$ only entails regular rules, we know that $\mathcal{R}_{t_i}\cap \mathcal{R}_{t_j}=\emptyset$ for $i\neq j$. We thus have $n_s\geq n_{t_1}+\ldots+n_{t_p}$. From the induction hypothesis we know that $b_u^+(O_{t_j,i})\leq n_{t_j}$. Furthermore, from the characterisation of composition we know $b_u^+(O_{t_1,i}\diamond \ldots \diamond O_{t_p,i})\leq b_u^+(O_{t_1,i}) +\ldots + b_u^+(O_{t_p,i}) \leq n_{t_1}+\ldots+n_{t_p}\leq n_s$. 
We clearly also have $b_u^+(X_{s,i})\leq n_s$. We conclude $b_u^+(O_{s,i})\leq n_s$.
\end{proof}

\begin{lemma}\label{lemmaXboundImpliesRuleA}
Suppose $s\notin \{r_1,\ldots,r_k\}$ and $j_s<k+1$. We have $b_x^-(O_{s,i}) > j_s$ unless $\mathcal{R}_s = \{r_{j_s},r_{j_s+1},\ldots,r_{j_s+n_s-1}\}$ and $\mathcal{K}\models r_{j_s}\circ r_{j_s+1}\circ\ldots\circ r_{j_s+n_s-1} \subseteq s$.
\end{lemma}
\begin{proof}
We show the result by structural induction, taking advantage of the acyclic nature of the rule base (i.e.\ Lemma \ref{lemmaCompositionEmbeddingAcyclic}). Let us first consider the base case where there are no rules of the form $t_1\circ \ldots\circ t_p \subseteq s$ entailed by $\mathcal{K}$ (except for the trivial $s\subseteq s$). If $s=r_j$ we have $O_{s,i}=X_{j,1}$ and thus clearly $b_x^-(O_{s,i})=j=j_s$. If $s\notin \{r_1,\ldots,r_k\}$ we have  $b_x^-(O_{s,i})=k+1$ and the result is thus also satisfied.

For the inductive step, let us now assume that for every rule of the form $t_1\circ \ldots\circ t_p\subseteq s$ entailed by $\mathcal{K}$ (apart from the trivial rule $s\subseteq s$) the result is already known to hold for $t_1,\ldots,t_p$. Since we assumed $s\notin \{r_1,\ldots,r_k\}$, we can only have $b_x^-(O_{s,i}) = j_s$ if there is some rule $t_1\circ \ldots\circ t_p\subseteq s$ entailed by $\mathcal{K}$ such that $b_x^-(O_{t_1}\diamond \ldots \diamond O_{t_p})=j_s$. Let us consider such a rule.
Because $\mathcal{K}$ only entails regular rules, it must be the case that $\mathcal{R}_{t_i}\cap \mathcal{R}_{t_j}=\emptyset$ for $i\neq j$. If $b_x^-(O_{t_1,i})>j_s$ then we clearly have $b_x^-(O_{t_1,i}\diamond \ldots \diamond O_{t_p,i})>j_s$. Let us therefore assume that $b_x^-(O_{t_1,i})=j_s$, which by induction implies $\mathcal{R}_{t_1}=\{r_{j_s},\ldots,r_{j_s+n_{t_1}-1}\}$ and $\mathcal{K}\models r_{j_s}\circ\ldots\circ r_{j_s+n_{t_1}-1} \subseteq t_1$. From the characterisation of composition, we know that  $b_x^-(O_{t_1,i}\diamond O_{t_2,i})=b_x^-(O_{t_1,i})$ can only hold if $b_x^-(O_{t_1,i})\geq b_x^-(O_{t_2,i}) - b_u^+(O_{t_1,i})$. From Lemma \ref{lemmaCompositionBoundsUandX} we know that $b_u^+(O_{t_1,i})\leq n_{t_1}$. Since we moreover assumed $b_x^-(O_{t_1,i})=j_{s}$ we find $b_x^-(O_{t_2,i})\leq n_{t_1}+j_{s}$. From Lemma \ref{lemmaCompositionBoundsUandX} we know that this means $j_{t_2}\leq n_{t_1}+j_{s}$. Since $\{r_{j_s},\ldots,r_{j_s+n_{t_1}-1}\}\cap \mathcal{R}_2=\emptyset$, this is only possible if $j_{t_2}=j_s+n_{t_1}$. By the induction hypothesis, this can only hold if $\mathcal{R}_{t_2}=\{r_{j_s+n_{t_1}},\ldots,r_{j_s+n_{t_1}+n_{t_2}-1}\}$ and $\mathcal{K}\models r_{j_s+n_{t_1}}\circ\ldots\circ r_{j_s+n_{t_1}+n_{t_2}-1}\subseteq t_2$. Continuing in the same way, we find that $b_x^-(O_{t_1,i}\diamond \ldots \diamond O_{t_p,i})=j_s$ is only possible if $\mathcal{K}$ entails $r_{j_s}\circ\ldots\circ r_{j_s+n_{t_1}-1} \subseteq t_1$, $r_{j_s+n_{t_1}}\circ\ldots\circ r_{j_s+n_{t_1}+n_{t_2}-1}\subseteq t_2$, \ldots, $r_{j_s+n_{t_1}+\ldots+n_{t_{p-1}}}\circ\ldots\circ r_{j_s+n_{t_1}+\ldots+n_{t_p}-1}\subseteq t_p$ and thus we also have $\mathcal{K}\models r_{j_s}\circ r_{j_s+1}\circ\ldots\circ r_{j_s+n_s-1} \subseteq s$.
\end{proof}

Let $s_1\circ \ldots \circ s_m \subseteq r$ be a rule which is entailed by $\mathcal{K}$. Let us associate with each $s_\ell$ a set of relations $\mathcal{R}_\ell\subseteq \{r_1,\ldots,r_k\}$ as follows: $r_j\in\mathcal{R}_\ell$ if $\mathcal{K}$ entails a rule with $r_j$ in the body and $s_\ell$ in the head.

\begin{lemma}\label{lemmaConstructonCompositionNoDistinctRules}
Let $r_1,\ldots,r_{k},r\in \mathcal{R}$ be distinct relations and assume that $\mathcal{K} \not\models r_1 \circ \ldots \circ r_k \subseteq r$. Then $O_{r_1,i}\diamond \ldots\diamond O_{r_k,i} \not\subseteq O_{r,i}$ for some coordinate $i$ for $\alpha_i$ the assignment defined above.
\end{lemma}
\begin{proof}
First note that $k< |\mathcal{R}|$ since $r_1,\ldots,r_{k+1}$ were assumed to be all distinct.
We clearly have:
\begin{align*}
O_{r_1,i}\diamond \ldots \diamond O_{r_k,i} 
& \supseteq A_{1,1}\diamond \ldots \diamond A_{k,1}  = A_{1,k}
\end{align*}
In particular, we have that $(1,k+1)\in O_{r_1,i}\diamond \ldots \diamond O_{r_k,i}$. To conclude the proof we show that $(1,k+1)\notin O_{r,i}$ and in particular that $b_u^+(O_{r,i})<k$. 
Given that $\mathcal{K} \not\models r_1 \circ \ldots \circ r_k \subseteq r$ we find from Lemma \ref{lemmaXboundImpliesRuleA} that $b_x^-(O_{r,i})\geq 2$. Since we trivially have $b_y^+(O_{r_,i}) \leq k+1$, we obtain $b_u^+(O_{r_,i})\leq k-1$.
\end{proof}

Lemma \ref{lemmaConstructonCompositionNoDistinctRules} shows that $\gamma_{\mathcal{K}}$ does not capture any unwanted rules of the form $r_1\circ \ldots\circ r_k \subseteq r$ where $r_1,\ldots,r_k,r$ are distinct relations. We now show that the same is true for rules where $r_1,\ldots,r_k,r$ are not necessarily distinct. Note that when $r_1,\ldots,r_k,r$ are not all distinct, we always have $\mathcal{K}\not\models r_1\circ \ldots\circ r_k\subseteq r$ (except for the trivial rule $r\subseteq r$), given our assumption about $\mathcal{K}$.

\begin{lemma}\label{lemmaConstructonCompositionNoNonDistinctRules}
Let $s_1,\ldots,s_{l} \in \mathcal{R}$ be such that $s_p=s_q$ for some $p\neq q$. There is a coordinate $i$ such that $O_{s_1,i}\diamond \ldots\diamond O_{s_l,i}\not \subseteq O_{r,i}$ for any relation $r\in\mathcal{R}$.
\end{lemma}
\begin{proof}
Let $r_1,\ldots,r_k$ be the unique relations among $s_1,\ldots,s_l$, i.e.\ we have $\{r_1,\ldots,r_k\}=\{s_1,\ldots,s_l\}$ with $k<l$.
Let us define the assignment $\alpha_i$ as follows:
\begin{align*}
\alpha_i(r_j) = (1,|\mathcal{R}|+1)
\end{align*}
and for $r\notin \{r_1,\ldots,r_k\}$ we have
\begin{align*}
\alpha_i(r) = (|\mathcal{R}|+1,|\mathcal{R}|+1)
\end{align*}
Clearly, we have that $b_u^+(O_{s_1,i}\diamond \ldots\diamond O_{s_l,i})\geq \min(l,|\mathcal{R}|)$. To complete the proof, we show that $b_u^+(O_{r,i})\leq \min(l-1,|\mathcal{R}|-1)$. Let us define $\mathcal{R}_s$ and $n_s$ as before. We show that $b_u^+(O_{r,i})\leq n_s$, from which the result follows since we clearly have $n_r \leq \min(k,|\mathcal{R}|-1)\leq \min(l-1,|\mathcal{R}|-1)$.

In the base case where there are no non-trivial rules $t_1\circ \ldots\circ t_p\subseteq r$ entailed by $\mathcal{K}$, we clearly have $b_u^+(O_{r,i}) = 0$ if $r\notin \{r_1,\ldots,r_k\}$ and thus $n_s=0$ (due to the assumption of there not being any rules $t_1\circ \ldots\circ t_p\subseteq r$ entailed by $\mathcal{K}$). If $r\in \{r_1,\ldots,r_k\}$ we have $n_r=1$ and $b_u^+(O_{r,i})=b_u^+(X_{r,i})=1$. For the inductive case, we find $b_u^+(X_{r,i})\leq n_r$ as in the base case. If $t_1\circ \ldots\circ t_p\subseteq r$ is entailed by $\mathcal{K}$ we know that $n_{t_1}+\ldots+n_{t_p}\leq n_{r}$, since $\mathcal{K}$ only entails regular (non-trivial) rules. By induction and the characterisation of composition, we know that $b_u^+(O_{t_1,i}\diamond \ldots\diamond O_{t_p,i})\leq b_u^+(O_{t_1,i}) +\ldots + b_u^+(O_{t_p,i}) =n_{t_1}+\ldots+n_{t_p}\leq n_{r}$.
\end{proof}

\begin{lemma}\label{lemmaConstructonCompositionNoNonDistinctRules2}
Let $s_1,\ldots,s_l,r \in \mathcal{R}$ be such that $r\in \{s_1,\ldots,s_l\}$ and $l\geq 2$. There is a coordinate $i$ such that $O_{s_1,i}\diamond \ldots\diamond O_{s_l,i}\not \subseteq O_{r,i}$ for any relation $r\in\mathcal{R}$.
\end{lemma}
\begin{proof}
Let $r_1,\ldots,r_k$ be the unique relations among $s_1,\ldots,s_l$, i.e.\ we have $\{r_1,\ldots,r_k\}=\{s_1,\ldots,s_l\}$ with $k\leq l$. Let us consider the same assignment $\alpha_i$ as in the proof of Lemma \ref{lemmaConstructonCompositionNoNonDistinctRules}. We then again find
that $b_u^+(O_{s_1,i}\diamond \ldots\diamond O_{s_l,i})\geq \min(l,|\mathcal{R}|)$ and  $b_u^+(O_{s,i})\leq n_s$ for each relation $s$. By construction, we have $b_u^+(X_{r,i})=0$. For each non-trivial rule of the form $t_1\circ \ldots \circ t_p \subseteq r$ we have $b_u^+(O_{t_j,i})\leq n_{t_j}$ and $n_{t_1}+\ldots+n_{t_p}\leq n_r$. Moreover, due to the fact that $\mathcal{K}$ only entails regular (non-trivial) rules, we have $r\notin \mathcal{R}_{t_j}$, meaning that $n_{t_1}+\ldots+n_{t_p} < n_r$. It follows that $b_u^+(O_{t_1,i}\diamond \ldots \diamond O_{t_p,i}) < n_r \leq k\leq \min(l,|\mathcal{R}|)$. Given that $l\geq 2$ we also have $b_u^+(X_{r,i})=1<\min(l,|\mathcal{R}|)$. We thus find $b_u^+(O_{t_1,i} < \min(l,|\mathcal{R}|)$, from which the result follows.
\end{proof}

\begin{proposition}
Let $\mathcal{K}$ be a set of regular composition rules. Assume that any extended composition rule entailed by $\mathcal{K}$ is either a trivial rule of the form $r\subseteq r$ or a regular rule. There exists an octagon embedding $\gamma$ which satisfies $\mathcal{K}$, and which only satisfies those extended composition rules which are entailed by $\mathcal{K}$.
\end{proposition}

\begin{proof}
The result follows directly from Lemmas \ref{lemmaConstructionCompositionRulesIsSound}, \ref{lemmaConstructonCompositionNoDistinctRules},  \ref{lemmaConstructonCompositionNoNonDistinctRules} and \ref{lemmaConstructonCompositionNoNonDistinctRules2}. 
\end{proof}

\section{Experimental Results}

We conducted two experiments: link prediction in high dimension, to compare with results previously published on this task, and link prediction in low dimension. These experiments aim to compare the performances of octagon embeddings with other region based embedding models.

\subsection{Implementation Details}

The octagon model was implemented using the PyKEEN library\footnote{\url{https://github.com/pykeen/pykeen}}. Octagon embeddings are learnt in batches using the Adam optimizer and a Self-Adversarial Negative Sampling (NSSA) loss, like BoxE and ExpressivE. Validation is performed every 10 epochs and the best model over 1,000 epochs is used for testing. In all experiments, octagons are initialised randomly with Xavier initialisation. We empirically observed that normalisation (Proposition~\ref{propNormalisation}) does not influence final performances. Octagons are thus unconstrained during training, for faster computation.

\begin{table*}[t]
    \centering
    \begin{tabular}{llccccccc}
        \toprule
        && dim. $n$ & $L_p$-norm & margin $\lambda$ & neg. samples $|N|$ & learning rate & batch size \\
        \midrule
        $uvxy$ & FB15k-237 & 1000 & 1 & 15 & 150 & $10^{-3}$ & 1024 \\
        $uvxy$ & WN18RR & 500 & 1 & 35 & 5 & $10^{-2}$ & 512 \\
        TransE & FB15k-237 & 40 & 1 & 4 & 50 & $10^{-3}$ & 256 \\
        TransE & WN18RR & 40 & 1 & 3 & 50 & $10^{-3}$ & 256 \\ 
        BoxE & FB15k-237 & 20 & 2 & 4 & 50 & $10^{-3}$ & 256 \\
        BoxE & WN18RR & 20 & 2 & 3 & 50 & $10^{-3}$ & 256 \\
        ExpressivE & FB15k-237 & 40 & 2 & 4 & 50 & $10^{-2}$ & 1024 \\
        ExpressivE & WN18RR & 40 & 2 & 3 & 50 & $10^{-2}$ & 512 \\
        $uvxy$ & FB15k-237 & 40 & 1 & 3 & 50 & $10^{-3}$ & 1024 \\
        $uvxy$ & WN18RR & 40 & 1 & 3 & 5 & $10^{-2}$ & 512 \\
        \bottomrule
    \end{tabular}
    \caption{Hyperparameters used in experiments. All configurations of the octagon model share the same hyperparameters.}
    \label{tab:hyperparameters}
\end{table*}

The full implementation, including configuration files to reproduce results discussed in Section~\ref{sec:ExpResults}, is available online\footnote{\url{https://github.com/vcharpenay/uvxy}}. Hyperparameters used in the experiments are also reproduced in Table~\ref{tab:hyperparameters} for completeness. Some hyperparameters (embedding dimension, number of negative samples and batch size) alter the inductive power of a model. For a fair comparison, we chose the same configuration as past experiments for these hyperparameters. In high dimension, we took ExpressivE experiments\footnote{\url{https://github.com/AleksVap/ExpressivE/}} as a basis, up to one exception: on WN18RR, we empirically observed that less negative samples gave slightly better results. We therefore set $|N| = 5$ instead of $|N| = 100$. In low dimension, we took hyperparameters used on MuRE\footnote{\url{https://github.com/ibalazevic/multirelational-poincare/}}, with the same exception on WN18RR. To our knowledge, no result had been made available for BoxE and ExpressivE in low dimension. Other hyperparameters that are rather influenced by the model's scoring function (margin, $L_p$-norm, learning rate) have been chosen after a grid search. Possible values were $\{ 3, 4, 6, 9, 12, 15, 21, 27, 35, 41 \}$ for the margin, $\{ 1, 2 \}$ for the $L_p$-norm and $\{ 10^{-2}, 10^{-3} \}$ for the optimizer's learning rate. In low dimension, a search over $\{ 256, 512, 1024 \}$ has also been performed for the optimal batch size.

\subsection{Discussion}
\label{sec:ExpResults}

\begin{table}[t]
    \centering
    \footnotesize
    \setlength\tabcolsep{2pt}
    \begin{tabular}{lcccccccc}
        \toprule
        & \multicolumn{4}{c}{FB15k-237} & \multicolumn{4}{c}{WN18RR} \\
        \cmidrule(lr){2-5}\cmidrule(lr){6-9}
        & H@1 & H@3 & H@10 & MRR & H@1 & H@3 & H@10 & MRR \\
        \midrule
        TransE & 23.3 & 37.2 & 53.1 & 33.2 & 01.3 & 40.1 & 52.9 & 22.3 \\        
        BoxE & 23.8 & 37.4 & 53.8 & 33.7 & 40.0 & 47.2 & 54.1 & 45.1 \\
        ExpressivE & 24.3 & 36.6 & 51.2 & 33.3 & 46.4 & 52.2 & 59.7 & 50.8 \\
        \midrule
        $u$ & 23.1 & 37.3 & 53.2 & 33.2 & 01.6 & 39.9 & 51.5 & 22.0 \\
        $ux$ & 23.3 & 37.1 & 52.5 & 33.1 & 01.9 & 39.0 & 51.6 & 21.8 \\
        $uxy$ & 23.2 & 37.2 & 53.1 & 33.2 & 01.8 & 40.8 & 52.8 & 22.8 \\
        $uv$ & 22.7 & 35.8 & 51.2 & 32.2 & 40.4 & 46.2 & 51.3 & 44.4 \\
        $uv^*$ & 24.1 & 36.9 & 52.8 & 33.5 & 43.6 & 48.5 & 52.9 & 46.9 \\
        $uvxy$ & 15.9 & 28.0 & 42.9 & 24.9 & 39.6 & 45.9 & 50.7 & 43.8 \\
        $uvxy^*$ & 24.1 & 36.7 & 51.7 & 33.2 & 43.6 & 49.2 & 56.1 & 47.9 \\
        $u+v$ & 24.0 & 37.2 & 53.0 & 33.6 & 42.4 & 46.2 & 51.8 & 45.5 \\
        \bottomrule
    \end{tabular}
    \caption{Link prediction performances of region based embedding models ($n = 1000$ on FB15k-237, $n = 500$ on WN18RR). Configurations with $^*$ use the variant with attention weights.}
    \label{tab:more-link-prediction}
\end{table}

Table~\ref{tab:more-link-prediction} gives results in high dimension. Results for TransE, BoxE and ExpressivE are not ours. They were respectively published in~\cite{DBLP:conf/iclr/SunDNT19},~\cite{DBLP:conf/nips/AbboudCLS20} and~\cite{DBLP:conf/iclr/0002S23}. Two results are available for the $uv$ and $uvxy$ configurations: with and without attention weights. Attention improves performance on all metrics. In the light of Propositions~\ref{propHexagonLimitation} and ~\ref{propFullyExpressive}, which highlight the importance of $v$ constraints, we also introduced a $u+v$ model consisting of a mixture a $u$ constraint in half of the coordinates and a $v$ constraint in the other half. This model clearly outperforms $u$ on WN18RR, especially on H@1, though remaining below $uv^*$ and $uvxy^*$ (with attention weights).

\begin{table}[t]
    \centering
    \footnotesize
    \setlength\tabcolsep{2pt}
    \begin{tabular}{lcccccccc}
        \toprule
        & \multicolumn{4}{c}{FB15k-237} & \multicolumn{4}{c}{WN18RR} \\
        \cmidrule(lr){2-5}\cmidrule(lr){6-9}
        & H@1 & H@3 & H@10 & MRR & H@1 & H@3 & H@10 & MRR \\
        \midrule
        TransE & 18.3 & 30.5 & 46.2 & 27.5 & 03.0 & 33.8 & 52.8 & 21.3 \\
        MuRE & 22.7 & 34.6 & 49.3 & 31.5 & 42.9 & 47.4 & 52.8 & 45.9 \\
        BoxE & 17.9 & 29.0 & 43.9 & 26.5 & 41.7 & 45.5 & 49.6 & 44.5 \\
        ExpressivE & 18.7 & 30.4 & 46.0 & 27.6 & 43.4 & 48.9 & 55.5 & 47.5 \\        
        \midrule
        $u$ & 20.3 & 32.8 & 48.9 & 29.6 & 02.4 & 36.9 & 49.1 & 21.0 \\
        $u^*$ & 18.1 & 29.3 & 43.3 & 26.6 & 22.2 & 40.1 & 51.5 & 32.9 \\
        $ux$ & 20.8 & 32.9 & 48.5 & 30.0 & 02.6 & 36.3 & 49.8 & 21.0 \\
        $uy$ & 20.3 & 32.8 & 48.5 & 29.6 & 02.5 & 38.0 & 49.8 & 21.6 \\
        $uxy$ & 20.5 & 32.9 & 48.5 & 29.7 & 02.4 & 37.1 & 49.0 & 21.1 \\
        $uv$ & 20.3 & 32.3 & 47.5 & 29.3 & 38.5 & 46.5 & 52.3 & 43.6 \\
        $uv^*$ & 19.6 & 31.1 & 45.1 & 28.2 & 41.1 & 47.0 & 51.1 & 44.9 \\
        $uvxy$ & 19.7 & 31.5 & 46.8 & 28.6 & 37.9 & 44.6 & 48.9 & 42.1 \\
        $uvxy^*$ & 20.4 & 31.6 & 46.3 & 28.9 & 35.6 & 43.8 & 49.1 & 40.6 \\
        $u+v$ & 20.2 & 32.3 & 48.0 & 29.3 & 40.9 & 46.3 & 50.5 & 44.5 \\
        \bottomrule
    \end{tabular}
    \caption{Link prediction of region based embedding models in low dimension ($n=40$). Configurations with $^*$ use the variant with attention weights.}
    \label{tab:link-prediction-40}
\end{table}

We evaluated more configurations in low dimension, as reported in Table~\ref{tab:link-prediction-40}. Results for MuRE were first published in~\cite{DBLP:conf/nips/BalazevicAH19}, others are ours.
From the table, we observe that attention weights do not always improve the performances of a model in low dimension contrary to higher-dimensional models. Conversely, most models derived from $uvxy$ outperform TransE, BoxE and ExpressivE on FB15k-237 in low dimension but not in high dimension.
We also note that $ux$ and $uy$ give comparable results, with a small penalty for $uy$.


\section*{Acknowledgments}

Computations have been performed on the supercomputer facilities of the Mésocentre Clermont-Auvergne of the Université Clermont Auvergne. Steven Schockaert was supported by grants from EPSRC (EP/W003309/1) and the Leverhulme Trust (RPG-2021-140).

\bibliographystyle{named}
\bibliography{ijcai24}

\end{document}